\documentclass[11pt,letterpaper]{article}%
\usepackage[T1]{fontenc}
\usepackage{libertine}
\usepackage[table]{xcolor}
\usepackage[american]{babel}
\usepackage{xfrac}
\usepackage{bbm}
\usepackage{authblk}
\usepackage{amsthm,amssymb}
\usepackage{bm}
\newtheorem{proposition}{Proposition}[section]
\newtheorem{theorem}[proposition]{Theorem}

\newtheorem{axiom}{Axiom}
\usepackage[format=plain,font=it]{caption}
\usepackage{algorithm}
\usepackage{algpseudocode}
\usepackage{pdflscape}
\usepackage{tabularx}
\usepackage{stmaryrd}

\usepackage{geometry}
\usepackage{setspace}

\usepackage{natbib}
\bibpunct[, ]{(}{)}{,}{a}{}{,}
\usepackage{mathtools}
\usepackage{booktabs}
\usepackage{multirow}
\usepackage{tikz}
\usetikzlibrary{positioning,calc,arrows.meta}

\usepackage{float}
\usepackage{subfig}
\usepackage{amsmath}
\usepackage{amsfonts}
\usepackage{comment}
\usepackage{dsfont}

\newtheorem*{definition}{Definition}

\definecolor{revnotebg}{HTML}{FFF3CD}
\definecolor{revnoteborder}{HTML}{B38600}
\newcommand{\revnote}[1]{}

\usepackage{array}
\newcolumntype{C}[1]{>{\centering\arraybackslash}m{#1}}
\newcolumntype{R}[1]{>{\raggedleft\arraybackslash}m{#1}}
\newcolumntype{L}[1]{>{\raggedright\arraybackslash}m{#1}}

\title{From Explaining to Enabling: Actionable Transparency in Algorithmic Insurance Pricing}
\date{}

\usepackage{graphicx,pstricks,psfrag}
\definecolor{bleu}{RGB}{0,101,189}
\definecolor{vert}{HTML}{004D40}
\definecolor{rose}{HTML}{D81B60}
\definecolor{bleuTOL}{HTML}{332288}
\definecolor{wongBlack}{RGB}{0,0,0}
\definecolor{wongGold}{RGB}{230, 159, 0}
\definecolor{wongLightBlue}{RGB}{86, 180, 233}
\definecolor{wongGreen}{RGB}{0, 158, 115}
\definecolor{wongYellow}{RGB}{240, 228, 66}
\definecolor{wongBlue}{RGB}{0, 114, 178}
\definecolor{wongOrange}{RGB}{213, 94, 0}
\definecolor{wongPurple}{RGB}{204, 121, 167}
\definecolor{colUncalibrated}{RGB}{191, 191, 191}
\definecolor{colRecalibrated}{RGB}{197, 214, 231}

\definecolor{bleuTOL}{HTML}{332288}
\definecolor{vertTOL}{HTML}{117733}
\definecolor{vertClairTOL}{HTML}{44AA99}
\definecolor{bleuClairTOL}{HTML}{88CCEE}
\definecolor{sableTOL}{HTML}{DDCC77}
\definecolor{parmeTOL}{HTML}{CC6677}
\definecolor{magentaTOL}{HTML}{AA4499}
\definecolor{roseTOL}{HTML}{882255}

\usepackage{hyperref}
\hypersetup{
    plainpages=false,
    bookmarksopen,
    bookmarksnumbered,
    unicode=true,
    pdftoolbar=true,
    pdfmenubar=true,
    pdffitwindow=false,
    pdfstartview={FitH},
    pdftitle={From Explaining to Enabling: Actionable Transparency in Algorithmic Insurance Pricing},
    pdfauthor={Agathe Fernandes Machado, Fei Huang},
    pdfkeywords={explainable AI, counterfactual explanations, algorithmic transparency, conformal prediction, insurance pricing, actionable recourse},
    pdfnewwindow=true,
    colorlinks=true,
    linkcolor=black,
    citecolor=black,
    filecolor=black,
    urlcolor=black,
}

\author[1]{Agathe~Fernandes~Machado\thanks{PhD candidate. Email: \href{mailto:fernandes_machado.agathe@courrier.uqam.ca}{fernandes\_machado.agathe@courrier.uqam.ca}}}
\author[2]{Fei~Huang\thanks{Associate Professor. Email: \href{mailto:feihuang@unsw.edu.au}{feihuang@unsw.edu.au}}}

\affil[1]{%
    \footnotesize Département de Mathématiques\\
    Université du Québec à Montréal\\
    Montréal (UQAM), Québec, Canada
}
\affil[2]{%
    \footnotesize School of Risk and Actuarial Studies, UNSW Business School, UNSW\\
    Sydney, New South Wales, Australia
}

\usepackage[misc]{ifsym}
\makeatletter
\def\@fnsymbol#1{%
   \ifcase#1\or
   \TextOrMath {\footnotesize\Letter} \dagger\or
   \TextOrMath {\footnotesize\Letter} \dagger\or
   \TextOrMath \textdaggerdbl \ddagger \or
   \TextOrMath \textsection  \mathsection\or
   \TextOrMath \textparagraph \mathparagraph\or
   \TextOrMath \textbardbl \|\or
   \TextOrMath {\textdagger\textdagger}{\dagger\dagger}\or
   \TextOrMath {\textdaggerdbl\textdaggerdbl}{\ddagger\ddagger}\else
   \@ctrerr \fi
}
\makeatother

\usepackage{fancyhdr}
\newcommand{\authornames}{\footnotesize\textsc{Fernandes Machado, Huang}}
\usepackage{etoolbox}
\makeatletter
\patchcmd{\NAT@test}{\else \NAT@nm}{\else \NAT@nmfmt{\NAT@nm}}{}{}
\DeclareRobustCommand\citepos
  {\begingroup
   \let\NAT@nmfmt\NAT@posfmt
   \NAT@swafalse\let\NAT@ctype\z@\NAT@partrue
   \@ifstar{\NAT@fulltrue\NAT@citetp}{\NAT@fullfalse\NAT@citetp}}
\let\NAT@orig@nmfmt\NAT@nmfmt
\def\NAT@posfmt#1{\NAT@orig@nmfmt{#1's}}
\makeatother

\begin{document}

\maketitle
\thispagestyle{empty}%

\begin{abstract}
Machine learning (ML) systems increasingly make consequential decisions about individuals, including credit, hiring, and insurance. Research on ML-driven insurance pricing has asked how the pricing mechanism itself can be made accountable. Yet it remains unclear how an individual policyholder's algorithmically determined premium should be explained to them. We argue that a useful explanation must be valid, actionable, and reliable. A \emph{valid} recommendation lowers the premium according to the model. An \emph{actionable} recommendation changes only characteristics the policyholder can control. A \emph{reliable} recommendation has an actual outcome that is likely to meet the target, even in sparsely observed regions of the feature space where the model is uncertain. We formalize these three requirements as axioms and develop a method that turns an algorithmic premium into recommendations that meet them. The method meets the third requirement with a Conformal Prediction constraint, which screens out recommendations that the model is too uncertain about and thereby bounds the rate of unreliable recommendations. We apply four Counterfactual Explanation techniques to 57,832 real Florida flood-insurance claims and 98,058 synthetic driver-telematics policies. We adapt these techniques for insurance applications, so that their recommendations change only controllable characteristics and resemble policyholders observed in the data. In our illustrations, the constrained recommendations for high-risk policyholders are both more reliable and associated with larger predicted risk reductions than the unconstrained ones. These results suggest that insurers can give policyholders explanations that are both actionable and trustworthy. The work contributes to research on algorithmic accountability and offers implications for insurers and regulators.
\end{abstract}

\begin{center}
\end{center}

\section{Introduction}
\label{sec:intro}

Machine learning (ML) models are increasingly embedded in consequential decisions about individuals, and insurance pricing is a leading example. Insurers use complex models, including tree ensembles and deep architectures, to set personalized premiums from data sources such as driver telematics \citep{Gao2018, kuo2021embeddingsattentionpredictivemodeling}. This makes pricing more accurate and more personalized, but it also makes the model's internal logic opaque to the individuals whose outcomes it determines. In Australia, for example, motor vehicle insurance premiums rose 8\% over the 2024--2025 financial year, and the Australian Securities and Investments Commission (ASIC) criticized insurers for failing to give customers clear, actionable information about what drove those increases \citep{asic2026consumers}.

The same tension recurs across data-driven organizations. Collective, model- or mechanism-level decision processes need to be linked back to individual-level understanding, accountability, and action, as in research on data-breach risk and firm-level managerial responsibility \citep{senborle2015, kashmiri2017}, individual privacy risk arising from collectively-mined data \citep{dwork2008, kifer2011, adjerid2018}, and the interpretability and fairness of AI governance more broadly \citep{rai2020, abbasi2018hbr}. \citet{zhangxu2024fairratemaking} examine this linkage problem for catastrophe insurance ratemaking. They ask how a catastrophe pool's collective risk load should be ``fairly'' allocated back to individual policies, and they derive a unique, Shapley-value-based allocation rule.

We study a complementary instance of the same phenomenon, one layer downstream of ratemaking. Whatever mechanism produced a premium, the policyholder who receives it needs an explanation that is understandable and actionable. Our research question is therefore as follows.

\begin{center}
\emph{RQ: How should an individual policyholder's algorithmically determined premium be explained to them, so that the explanation is both actionable and reliable?}
\end{center}

Research on Explainable Artificial Intelligence (XAI) in insurance pricing has mostly addressed transparency at the level of the model. It reports which features matter on average across a portfolio \citep{Delcaillau2022, henckaerts2020boostinginsightsinsurancetariff} or, for a single instance, how much each feature contributed to the prediction actually received \citep{NIPS2017_8a20a862}. Neither answers the question a policyholder is most likely to ask. What would need to change, and by how much, to obtain a lower premium? \citet{MILLER20191} argues from the social sciences that human explanation-seeking is inherently \emph{contrastive} in this sense. \citet{wachter2018counterfactual} operationalize this insight for automated decision-making as Counterfactual Explanation (CE), the smallest change to an individual's characteristics that would produce a different, preferred outcome, without requiring disclosure of a model's internal logic.

These techniques align with what \citet{wachter2018counterfactual} identify as the three purposes an explanation should serve a data subject, namely to understand a decision, to contest it, and to know what would need to change to alter it. Most XAI applied to insurance pricing addresses only the first. A Society of Actuaries report calls on practitioners to ensure that ``the model inputs and outputs can be explained to individuals impacted by the use of the model'' \citep{SOA2021}. In this paper, ``counterfactual'' refers to the fitted model and should not be interpreted as a causal claim. The generated counterfactual instance $\boldsymbol{x}'$ reflects statistical association between features and the model's output, as in most post-hoc XAI methods. It does not reproduce the causal relationships that a Structural Causal Model would capture \citep{pearl2009causality}. But, in some Counterfactual Explanation methods, it is possible to enforce known causal relationships between features \citep{Bobek2025ActionableAndDiverseCounterfactual, causalbiocf2025}.

A concrete example from our flood application (Section~\ref{subsec:demo-flood}) shows both the promise and the problem. For one policyholder, moving the insured contents to the upper floors lowers the model's predicted loss ratio from 0.46 to 0.31. Yet the model's 80\% upper bound for the predicted loss ratio after that change is 0.56, above the original 0.46 (Table~\ref{tab:flood-query-instance}), so the apparent reduction may reflect the model's uncertainty rather than an actual change in risk.

Two further considerations make this individual-level communication problem nontrivial rather than a straightforward application of existing CE methods. First, not every characteristic a CE technique might recommend changing is one an individual can act on. The algorithmic-recourse literature argues that recommendations should be restricted to controllable features \citep{linear2019ustun}. Recommending a change to an immutable characteristic such as gender is unhelpful, and it can also hold the individual responsible for who they are rather than for what they do.

Second, and specific to the reliability of recourse, behavioral evidence shows that people's willingness to rely on algorithmic advice is highly sensitive to whether that advice subsequently proves correct. People who observe an algorithm error even once become substantially less willing to trust it again \citep{dietvorst2015algorithm}. Absent such disconfirming evidence, algorithmic recommendations are often over-trusted relative to human judgment \citep{logg2019algorithm}. This is especially consequential for high-risk policyholders, who disproportionately, though not universally, lie in sparsely observed, high-uncertainty regions of the feature space. A CE technique's ordinary proximity objective can therefore recommend changes whose predicted premium reduction is an artifact of model uncertainty rather than of an actual reduction in risk.

\subsection{Research Objective and Contributions}

We develop an approach to insurer--policyholder communication that produces valid, actionable, and reliable Counterfactual Explanations (CEs) of an ML-based premium. Validity (Axiom~\ref{ax:validity} in Section~\ref{subsec:notation}) comes from the CE search itself, which looks for changes that bring the model's predicted premium into a predefined desired range. Actionability (Axiom~\ref{ax:actionability}) comes from restricting recommendations to features a policyholder can control. Reliability (Axiom~\ref{ax:reliability}) comes from screening each recommendation against the model's own uncertainty, so that a predicted premium reduction is not an artifact of a region where the model has little data.

To enforce reliability, we build on Conformal Prediction (CP) \citep{vovk2022algorithmic}. For any underlying model, CP constructs a prediction interval that contains the actual outcome with at least a chosen probability, on average over the population. We formulate a CP-constrained CE search for the regression models used in insurance, which predict claim frequency (how often claims occur) and severity (how large they are). We show formally that this constraint bounds the rate at which profiles (combinations of characteristics) are accepted as reliable yet have an actual outcome above the target. The bound holds when the generated instances $\boldsymbol{x}'$ produced by CE techniques are drawn from the same population as the hold-out calibration dataset used to estimate CP intervals in practice. We also characterize the proximity cost that this constraint imposes. In our illustrations, the constrained recommendations are more trustworthy and are associated with larger predicted risk reductions.

This paper makes three contributions to the Information Systems (IS) literature. First, we contribute to the literature on algorithmic accountability, which concerns whether individuals affected by an algorithmic decision can understand, contest, and act on it. We show that an explanation can be valid and actionable and still fail the individual, because its predicted benefit may not materialize. We formalize this requirement as reliability and show that it does not follow from validity or actionability. Prior work addresses it only for classification and without actionability constraints. Second, we contribute to the literature on XAI and algorithmic transparency by showing how reliability can be enforced and what it costs for regression models. A Conformal Prediction constraint bounds the rate of unreliable recommendations. The constraint can never reduce, and may increase, how much a policyholder must change. Third, we contribute to research on AI-enabled insurance pricing by proposing explanations tailored to individual policyholders rather than to regulators or for insurance company governance. We show, with flood-insurance claims and driver-telematics data, that most valid recommendations are not reliable. In the telematics test sample, only 0.2\% to 3.8\% of the recommendations of the four techniques meet the reliability constraint at 80\% confidence, and the techniques differ in how readily they produce reliable recommendations. Insurers can use these results to choose among techniques and to judge when a recommendation can be trusted.

The remainder of the paper proceeds as follows. Section~\ref{sec:related-work} reviews related literature. Section~\ref{sec:ce-insurance} develops the formal approach, from the axioms and the four CE techniques to the CP constraint and its properties. Section~\ref{sec:demonstration} applies the approach to real flood-insurance claims data and a synthetic driver-telematics dataset. Section~\ref{sec:discussion} discusses implications and limitations, and Section~\ref{sec:conclusion} concludes. Proofs and further technical and experimental details are in the appendices.

\section{Related Work}
\label{sec:related-work}

We draw on three streams of research. The first concerns individual-level accountability for algorithmic decisions; the second, explainability in insurance pricing; and, the third, Counterfactual Explanations and their treatment of uncertainty.

\subsection{Individual-Level Accountability for Algorithmic Decisions}
\label{subsec:related-is}

The need to link model-level decisions back to individual-level understanding and action recurs across a broad range of IS phenomena and predates the current wave of XAI research by several decades. We refer to this need as individual-level accountability for algorithmic decisions, meaning that affected individuals can understand, contest, and act on them. \citet{gregorbenbasat1999} lay the theoretical foundations for explanation facilities in intelligent systems. They argue that explanations serve distinct functions, including learning, increasing user confidence and acceptance, and enabling users to judge when a system's advice should and should not be followed, and that the appropriate content and form of an explanation depend on the task the user faces. Their theory anticipates the distinction between explanations that build understanding and explanations that enable action. Their emphasis on matching explanation form to user task is a direct antecedent of the contrastive, actionable framing we adopt from the more recent XAI literature.

Behavioral research examines how people update their trust in algorithmic judgment. \citet{dietvorst2015algorithm} show that even a single observed algorithmic error sharply reduces subsequent reliance, a phenomenon termed algorithm aversion. \citet{logg2019algorithm} show that, absent such evidence, people often over-rely on algorithmic relative to human judgment. Procedural-justice research shows that perceptions of fair treatment depend on whether decisions are explained in a way that acknowledges the affected party's ability to respond \citep{lindtyler1988}. \citet{MILLER20191} synthesizes evidence that explanation-seeking is inherently contrastive, and the algorithmic-recourse literature formalizes actionability as a design requirement that distinguishes controllable from immutable characteristics \citep{linear2019ustun}. This body of work implies that an individual-level explanation is incomplete if it is not both contrastive and actionable, and that it can undermine institutional trust if its reliability is not communicated or ensured.

Insurance pricing is a suitable setting for three reasons. First, regulation motivates explanation, although its legal force varies by jurisdiction and insurance line. The GDPR protects data subjects against solely automated decisions with legal or similarly significant effects (Article 22), and Recital 71 has been read, although the reading remains debated, as supporting a right to an explanation \citep{GDPR2016}. The EU AI Act classifies AI systems used for risk assessment and pricing as high-risk, but only for life and health insurance \citep{eu_ai_act_annex3}. The GDPR is therefore the more directly applicable motivation for our property and casualty applications.

Second, the stakes of actionable recourse are high for high-risk policyholders, given climate-driven affordability pressure in lines such as flood and wildfire insurance \citep{affordability2016usreports}, where mitigation actions are known to lower modeled catastrophe risk \citep{brinkmann2022catastrophe}.

Third, claim frequency and severity are modeled by regression, a setting often overlooked in the CE and algorithmic-recourse literatures, which have been developed primarily for binary decisions such as loan approval.

\subsection{Machine Learning-Driven Insurance Pricing}

Research applying XAI to insurance pricing models, largely published in actuarial science, has focused mostly on insurer--regulator communication and internal model governance \citep{xin2024trustxai, Kuo2023Towards, Delcaillau2022, risks9070126}. It reports feature rankings, importance values, or visual diagnostics that describe the model and say nothing about what an individual could do differently. \citet{henckaerts2020boostinginsightsinsurancetariff} discuss transparency for all pricing-model stakeholders, including policyholders, but rely on global variable-importance and partial dependence plots. A related thread develops regression-based fairness criteria aligned with antidiscrimination regulation \citep{Xin02042024}. That work asks whether a prediction is statistically fair across protected groups. We ask whether the individual who received the prediction can understand and act on it, a problem that persists even when a pricing model satisfies such criteria.

The contribution of \citet{zhangxu2024fairratemaking} operates at the level of the risk-load allocation mechanism. They show, drawing an analogy to fair data valuation in ML \citep{Ghorbani2019}, that a Shapley-value-based allocation is the unique method satisfying three normatively motivated axioms, and that several allocation methods used in practice have disparate impact on minority-owned properties. We study a problem one layer downstream, namely how the premium that such a mechanism produces should be explained to the policyholder who receives it. We return to the relationship between these two layers in Section~\ref{sec:discussion}.

\subsection{Why Counterfactual Explanations}
\label{subsec:why-ce}

Section~\ref{subsec:related-is} argued that an individual-level explanation is incomplete if it is not contrastive and actionable. To show why Counterfactual Explanation (CE) fits this purpose, we place it within the taxonomy of ML interpretability methods \citep{molnar2025}. We require a post-hoc method, because insurers are unlikely to replace validated, higher-performing pricing models with intrinsically interpretable ones solely to communicate with policyholders. We require a model-agnostic method, so that the approach generalizes across insurers and model choices, and a local method, because our objective is to explain one policyholder's premium to that policyholder. Many methods meet these three requirements, including SHAP and LIME. Appendix~\ref{sec:ec-interpretability} compares them and shows that they describe the prediction actually received, and are neither contrastive nor restricted to features the individual can control. Appendix~\ref{sec:appendix-flood-data} illustrates the difference for one flood policyholder. We therefore adopt CE as the basis for the approach developed in Section~\ref{sec:ce-insurance}. CE was not originally designed, however, for the regression, actionability-constrained, uncertainty-aware form that the insurance setting requires.

\subsection{Validity and Reliability of Individual-Level Recommendations}
\label{subsec:reliability-notions}

A CE recommendation can satisfy its own optimization objective and still fail the policyholder. The recommended change may lower the model's predicted premium without lowering the outcome the policyholder would actually experience. We therefore distinguish two properties that CE communication tends to treat as one, which correspond to Axioms~\ref{ax:validity} and~\ref{ax:reliability} in Section~\ref{subsec:notation}. \emph{Recommendation validity} asks whether following the recommended change would, according to the fitted model $h$, produce the stated lower prediction, and any CE technique that solves its optimization problem exactly guarantees this by construction. \emph{Recommendation reliability} asks whether that predicted reduction can be trusted, meaning that the actual outcome at the recommended change is likely to fall in the target range. Validity concerns the model's output, and reliability concerns the evidence behind that output. CE techniques often search for the closest instance satisfying validity, and high-risk policyholders, for whom recourse is most consequential, disproportionately though not exclusively occupy sparse, high-uncertainty regions. Validity and reliability can therefore diverge where the stakes are highest.

\citet{ChristodoulouSun2026arxiv2602} show, studying CE sensitivity to both aleatoric 
and epistemic uncertainties,  
that prediction error alone is an insufficient basis for assessing a counterfactual's reliability, reinforcing the need for an explicit uncertainty-aware constraint rather than relying on point-prediction accuracy. Recent work incorporates predictive uncertainty into CE search for classification tasks using Conformal Prediction \citep{cpaaai2024ce, adams2025individualisedcounterfactualexamplesusing, bilkhoo2025confexuncertaintyawarecounterfactualexplanations} and, relatedly, by favoring counterfactuals in high training-density regions \citep{face2020, sequentialCE2025, ijcai2022p104, ijcai2020p395}. 

A related literature, Explainable Outlier Detection, flags anomalous instances and determines how they could be modified to become inliers \citep{sulem2022diversecounterfactualexplanationsanomaly, ZHOU2025111532, subspacedensity2023} (see \citealp{surveyanomaly2023} for a survey). High-risk policyholders often resemble such outliers, but our constraint does not require a flagging step, and it pursues a reliable premium reduction rather than the restoration of typicality. Closer to our setting, concurrent work generates counterfactuals for conformal regression intervals \citep{maalej2026conformal}, evaluating them on plausibility and prediction difficulty and not on actionability or a reliability property. We adopt the Conformal Prediction (CP) approach and extend it to the actionability-constrained regression setting required for insurance frequency and severity models, formalized in Section~\ref{sec:ce-insurance}.

\section{Counterfactual Explanations for Insurance Pricing}
\label{sec:ce-insurance}

This section formalizes what actionable and reliable recourse requires. A Counterfactual Explanation (CE) tells a policyholder which changes to controllable characteristics would bring the model's predicted outcome into a target range. We require three things of such a recommendation. The model must predict that it works (validity), it must change only what the policyholder can change (actionability), and the prediction must be trustworthy, meaning that the actual outcome is also likely to be in the target range (reliability). Existing techniques can be made to meet the first two, but nothing in them addresses the third. This matters most for high-risk policyholders, and we show how a Conformal Prediction (CP) constraint addresses it, at the price of recommendations that tend to change more features. We first define the notation and the three axioms, then compare candidate techniques, illustrate the gap between validity and reliability. To address this gap, we develop the CP constraint, and show how it enters each CE technique.

\subsection{Notation and Definitions}
\label{subsec:notation}

CE techniques were developed primarily for classification tasks, where a query instance $\boldsymbol{x}$ with an unfavorable predicted label seeks a counterfactual instance $\boldsymbol{x}'$ with a flipped, favorable label \citep{8807255, russell2029diverseCE, face2020}. Applying CE to insurance frequency and severity models, which are continuous or count-valued regression outcomes, requires generalizing the target from a single label $y^c$ to a desired range $Y'$. As discussed in Section~\ref{subsec:related-is}, it also requires restricting recommendations to controllable features, which certain existing CE implementations do not natively support. We formalize these requirements below. Throughout, a \emph{recommendation} is a counterfactual instance presented to a policyholder as a suggested set of changes.

Let $h : \mathcal{X} \to \mathcal{Y}$, with $\mathcal{X} \subset \mathbb{R}^d$ and $\mathcal{Y} \subset \mathbb{R}_+$, denote a frequency or severity model mapping a policyholder's characteristics $\boldsymbol{X}$ to a predicted outcome $Y$. Let $\boldsymbol{x} \in \mathbb{R}^d$, drawn from the distribution of $\boldsymbol{X}$, denote the observed characteristics of a query policyholder, and let $\mathcal{D}_{\text{train}} = \left(\boldsymbol{X}_i, Y_i\right)_{i=1}^{n_{\text{train}}}$, $\mathcal{D}_{\text{cal}}= \left(\boldsymbol{X}_i, Y_i\right)_{i=1}^{n_{\text{cal}}}$, $\mathcal{D}_{\text{test}}= \left(\boldsymbol{X}_i, Y_i\right)_{i=1}^{n_{\text{test}}}$ denote, respectively, the training (with $n_{\text{train}}$ observations), calibration (with $n_{\text{cal}}$ observations), and test sets (with $n_{\text{test}}$ observations).

\begin{definition}[Counterfactual Explanation for Regression]
A Counterfactual Explanation for the individual $\boldsymbol{x} \in \mathbb{R}^d$ consists of an instance $\boldsymbol{x}' \in \mathbb{R}^d$ that differs minimally from $\boldsymbol{x}$ such that its prediction $h(\boldsymbol{x}')$ belongs to a desired predefined range $Y'$. The minimal objective of CE techniques follows the optimization problem \citep{wachter2018counterfactual}
\begin{equation}
\label{eq:min-ce}
    \arg \min_{\boldsymbol{x}' \in \mathbb{R}^d} l\left(h(\boldsymbol{x}'), Y'\right) + C_1 \cdot \theta(\boldsymbol{x}', \boldsymbol{x}),
\end{equation}
with the loss \citep{Dandl_2020}
$$
l\left(h(\boldsymbol{x}'),Y'\right) := \begin{cases}
0, & h(\boldsymbol{x}') \in Y', \\
\min_{y' \in Y'} \left|h(\boldsymbol{x}')-y'\right|, & \text{otherwise},
\end{cases}
$$
where $\theta$ is a proximity penalty (e.g., the Gower distance, which handles mixed numerical and categorical features), and $C_1 \in \mathbb{R}_+^*$.
\end{definition}

Solutions to Equation~\ref{eq:min-ce} balance \emph{validity} ($h(\boldsymbol{x}') \in Y'$) and \emph{proximity} by construction. 
The broader CE literature incorporates further properties through additional regularization terms \citep{wachter2018counterfactual, dice2020, russell2029diverseCE, Guidotti2024Counterfactual}, and some work derives the CE optimization problem in closed form for GLMs \citep{artelt2019glm}. These properties include sparsity (the number of features changed, captured by an $L_0$ norm in $\theta$), plausibility, causality, multiplicity and diversity, and robustness.

These properties are quality dimensions that a practitioner trades off against one another, and several conflict, for example sparsity and plausibility. We call them \emph{desirable properties}. They are worth optimizing where possible, but, unlike the axioms introduced below, a procedure does not have to meet them to count as recourse at all. Sparsity is the number of features changed. Proximity is the distance between $\boldsymbol{x}'$ and $\boldsymbol{x}$ under $\theta$. Plausibility means that $\boldsymbol{x}'$ lies on the data manifold of $\mathcal{D}_{\text{train}}$, which we approximate by the Gower distance to the nearest training instance. It matters because the model's prediction at a point far from the training data is an extrapolation, so a recommendation there is less credible. Causality means that the recommended changes respect the causal relationships among features. Multiplicity means that several, mutually distinct counterfactual instances $\boldsymbol{x}'_1, \ldots, \boldsymbol{x}'_M$ are returned, and diversity is the spread among them. Robustness means that a recommendation remains valid under small perturbations of the query or the model. It is outside the scope of our comparison. Appendix~\ref{sec:ec-landscape} compares these properties across techniques. We isolate three axioms, the minimal requirements that a CE procedure must satisfy for a policyholder to be able to act on its output.

A \emph{CE procedure} $\mathcal{P}$ maps a model $h$, a query instance $\boldsymbol{x}$, and a target range $Y' = [0, y^u]$ to $M > 0$ instances $\boldsymbol{x}'_1, \ldots, \boldsymbol{x}'_M$. When $M > 1$, multiplicity is ensured, as more than one counterfactual is returned to the policyholder. The lower bound of $Y'$ is $0$ because insurance outcomes are non-negative. Let $\mathcal{I} \subseteq \{1,\ldots,d\}$ index the immutable features, and let $\mathcal{A}(\mathcal{P}) \subseteq \mathbb{R}^d$ denote the set of all counterfactual instances that $\mathcal{P}$ can return across all query policyholders. Let $(\boldsymbol{X}, Y)$ denote a policyholder drawn from the population independently of the data used to build $\mathcal{P}$, whose actual outcome $Y$ is not observed at the time of recommendation.

\begin{axiom}[Validity]
\label{ax:validity}
For all \(m \in \{1,\ldots,M\}\), the model's prediction at the recommended instance meets the target, $h(\boldsymbol{x}'_m) \in Y'$.
\end{axiom}

\begin{axiom}[Actionability]
\label{ax:actionability}
For all \(m \in \{1,\ldots,M\}\) and all $j \in \mathcal{I}$, $x'_{m,j} = x_j$.
\end{axiom}

\begin{axiom}[Reliability at level $\alpha$]
\label{ax:reliability}
The rate at which profiles in the range of the procedure have an actual outcome above the target is at most $\alpha \in (0,1)$, $\mathbb{P}\left(\boldsymbol{X} \in \mathcal{A}(\mathcal{P}),\; Y > y^u\right) \leq \alpha$.
\end{axiom}

Each axiom is a condition whose violation disqualifies a procedure as a source of recourse. A recommendation whose predicted outcome misses the target is not a counterfactual for that target (Axiom~\ref{ax:validity}). A recommendation that changes an immutable characteristic asks the policyholder to change who they are rather than what they do \citep{linear2019ustun} (Axiom~\ref{ax:actionability}). A recommendation drawn from a region in which the actual outcome frequently exceeds the target offers no recourse, because the policyholder acts on the change and the benefit may not materialize (Axiom~\ref{ax:reliability}). The first two are standard in the recourse literature. The third concerns the actual outcome rather than the model's prediction, and Section~\ref{subsec:unique-challenges} shows that the first two do not imply it. Axiom~\ref{ax:reliability} is informative only relative to the probability mass of $\mathcal{A}(\mathcal{P})$, because it implies $\mathbb{P}(Y \leq y^u \mid \boldsymbol{X} \in \mathcal{A}(\mathcal{P})) \geq 1 - \alpha / \mathbb{P}(\boldsymbol{X} \in \mathcal{A}(\mathcal{P}))$. While \emph{proximity} and \emph{plausibility} are not explicitly formalized in the axioms, we also account for them in the following, either as objectives in the optimization problem for CEs or as post-hoc filtering constraints.


\subsection{Comparison of Candidate Counterfactual Explanation Techniques}
\label{subsec:comparison}

Section~\ref{subsec:notation} separates three axioms, which are required, from desirable properties.
Furthermore, most existing CE techniques were developed for classification; however, in the context of insurance pricing, they must handle regression tasks to predict frequency and severity. Appendix~\ref{sec:ec-landscape} compares 13 representative techniques against the three axioms, the desirable properties, and applicability to regression. To our knowledge, none of them bounds the rate of unreliable recommendations (Axiom~\ref{ax:reliability}), and only MOC combines actionability, plausibility, multiplicity, and regression in its published form.

We therefore retain four techniques for the remainder of the paper. They represent four different search strategies, namely nearest observed instances, sequential modification, multi-objective genetic search, and diversity optimization. Public implementations of all four handle regression and return several counterfactuals per query (namely, \emph{multiplicity}), and all four can be made to satisfy Axiom~\ref{ax:actionability} with a modest extension. We extend the first two to enforce actionability, which their published implementations lack, and add a post-hoc plausibility filter to DiCE. WhatIf \citep{8807255} returns the closest matching training instances with predictions in $Y'$. NICE \citep{prototype2024brughmans} sequentially replaces individual feature values of $\boldsymbol{x}$ with values from a nearby valid instance. MOC \citep{Dandl_2020} uses a genetic algorithm (the non-dominated sorting genetic algorithm II, NSGA-II) to optimize the potentially conflicting objectives of validity, proximity, sparsity, and plausibility, without requiring convexity or differentiability \citep{goldberg1989genetic}. Pareto dominance, under which one candidate is preferred only if it is at least as good on every objective and better on one, removes the need to specify the objectives' relative weights in advance. DiCE \citep{dice2020} explicitly optimizes diversity across $M$ recommendations.


The four techniques solve different optimization problems. Once extended, all four satisfy Axiom~\ref{ax:actionability} and, subject to an implementated post-hoc filter for MOC, Axiom~\ref{ax:validity} (Table~\ref{tab:ce-properties}), but they achieve validity through different mechanisms. Combining the multiplicity property with the use of several techniques also gives a policyholder more counterfactual instances to consider, because the ease of implementing a change may vary across policyholders. Appendix~\ref{sec:ec-formulations} gives the optimization problem that each technique solves.


\begin{table}[!htbp]
\centering
\small
\setlength{\tabcolsep}{4pt}
\begin{tabular}{l c c c c c c c}
\toprule
& \textbf{Axiom~1} & \textbf{Axiom~2} & \multicolumn{2}{c}{\textbf{Axiom~3}} & \multicolumn{3}{c}{\textbf{Desirable properties}}\\
\cmidrule(lr){2-2}\cmidrule(lr){3-3}\cmidrule(lr){4-5}\cmidrule(lr){6-8}
Method & Valid. & Action. & \multicolumn{2}{c}{Reliab.} & Plausib. & Causal. & Multip.\\
\cmidrule(lr){4-5}
& & & Uncon. & CP filter & & &\\
\midrule
WhatIf \citep{8807255} & \checkmark & \checkmark$^\dagger$ & $\times$ & (\checkmark)$^\dagger$ & \checkmark & $\times$ & \checkmark \\
NICE \citep{prototype2024brughmans} & \checkmark & \checkmark$^\dagger$ & $\times$ & (\checkmark)$^\dagger$ & \checkmark & $\times$ & \checkmark \\
MOC \citep{Dandl_2020} & (\checkmark) & \checkmark & $\times$ & (\checkmark)$^\dagger$ & \checkmark & $\times$ & \checkmark \\
DiCE \citep{dice2020} & \checkmark & \checkmark & $\times$ & (\checkmark)$^\dagger$ & \checkmark$^\dagger$ & $\times$ & \checkmark \\
\bottomrule
\end{tabular}
\caption{The four retained Counterfactual Explanation techniques against the three axioms (validity, actionability, reliability) and three desirable properties (plausibility, causality, multiplicity), after our extensions. ``Uncon.'' denotes the unconstrained version of each technique. All four handle regression. \checkmark\ indicates that the property holds, (\checkmark) that it holds only through a post-hoc filter and not through the search, and $\times$ that it is not guaranteed. $\dagger$ marks a property added by this paper relative to the technique's published implementation (Appendix~\ref{sec:ec-landscape}), namely actionability for WhatIf and NICE and plausibility through a post-hoc filter for DiCE. The CP filter column applies the CP constraint of Equation~\ref{eq:cp-ce}, and may return no instance when the hard constraint is infeasible. Axiom~3 then holds by Theorem~\ref{thm:sufficiency}.}
\label{tab:ce-properties}
\end{table}

We extend the R implementations of WhatIf and NICE in the \texttt{counterfactuals} package \citep{JSSv115i09} so that candidates leave immutable features unchanged relative to $\boldsymbol{x}$. Appendix~\ref{sec:ec-formulations} describes the two extensions. For DiCE, actionability is natively supported by the underlying implementation, and we add a post-hoc plausibility filter retaining only counterfactuals within a bounded Gower distance of their nearest training neighbor.

No single technique dominates on all criteria. We therefore apply several techniques in Section~\ref{sec:demonstration} and report the effect of the reliability constraint introduced in Section~\ref{subsec:cp-constraint} across all of them.

Table~\ref{tab:method-features} summarizes the main features, strengths, and limitations of the four techniques. The strengths and limitations that rest on performance come from the telematics comparison in Section~\ref{subsubsec:comparison-telematics} and should be read as evidence from a single dataset.

\begin{table}[!htbp]
\centering
\small
\begin{tabularx}{\textwidth}{l >{\raggedright\arraybackslash}p{3.3cm} >{\raggedright\arraybackslash}X >{\raggedright\arraybackslash}X}
\toprule
Method & Search strategy & Main strengths & Main limitations\\
\midrule
WhatIf & Returns the $M$ closest training instances whose prediction lies in the target range & Plausible and valid by construction, and by far the fastest & Cannot recommend a change absent from the training data, and is less sparse than NICE and MOC\\
NICE & Starts from a nearby valid training instance and moves toward it one feature at a time & Sparse recommendations, valid by construction, and the highest share of reliable recommendations in our sample & Tied to its starting instance, and most of its recommendations are also found by WhatIf and MOC\\
MOC & Genetic algorithm (NSGA-II) that optimizes validity, proximity, sparsity, and plausibility jointly & Closest, sparsest, and most plausible recommendations in our sample, and can propose changes not observed in the data & Validity is not guaranteed by the search, very few recommendations are reliable, and the algorithm needs tuning\\
DiCE & Optimizes diversity across the $M$ recommendations & Largest and most diverse set of recommendations, and valid within the search & Highest distance and sparsity in our sample, needs a post-hoc plausibility filter, and is implemented in Python\\
\bottomrule
\end{tabularx}
\caption{Features, strengths, and limitations of the four retained CE techniques.}
\label{tab:method-features}
\end{table}


\subsection{The Validity-Reliability Gap for High-Risk Query Instances}
\label{subsec:unique-challenges}

The challenge motivating this paper arises for a high-risk query instance, one with a large observed claim frequency or severity, correspondingly located in a sparse, low-density region of the feature space. An unconstrained CE search for such an instance, driven purely by the proximity objective common to all CE techniques, tends to return counterfactuals that remain in the same sparse region surrounding the query instance, because proximity is minimized there. Yet this is the region in which the fitted model $h$ is least well-determined by data, so the model's prediction for the recommended counterfactual is itself least trustworthy. The policyholder most in need of actionable recourse is, by the structure of the problem, the one for whom an unconstrained recommendation is least likely to be reliable in the sense of Section~\ref{subsec:reliability-notions}.

We illustrate the gap with a synthetic example in which the data generating process (DGP) is known. Consider $n=510$ observations of two continuous features $X_1, X_2 \in \mathbb{R}$. Of these, 500 inliers are drawn from a standard bivariate normal distribution $\mathcal{N}(\boldsymbol{0}, I_2)$, and 10 feature-space outliers are drawn from $\mathcal{N}((4,4)^\top, 1.2^2 I_2)$, representing a small population of atypical, high-risk policyholders. For the inliers, the outcome $Y$ (an annual claim count) is Poisson with conditional mean $\mu(X_1,X_2) = \exp(-1+0.3X_1-0.4X_2)$. The outliers have their responses drawn from the upper tail of the inlier outcome distribution, which creates a stress test in which both features and outcomes are atypical. We fit a Poisson GLM $h$ on a 60\% training split and calibrate a prediction interval on a 20\% calibration split. Figure~\ref{fig:unique-challenge} applies the WhatIf technique to an outlier query instance from the remaining 20\% test split, with and without the reliability constraint of Section~\ref{subsec:cp-constraint}.

\begin{figure}[!htbp]
    \centering
    \includegraphics[width=.5\textwidth]{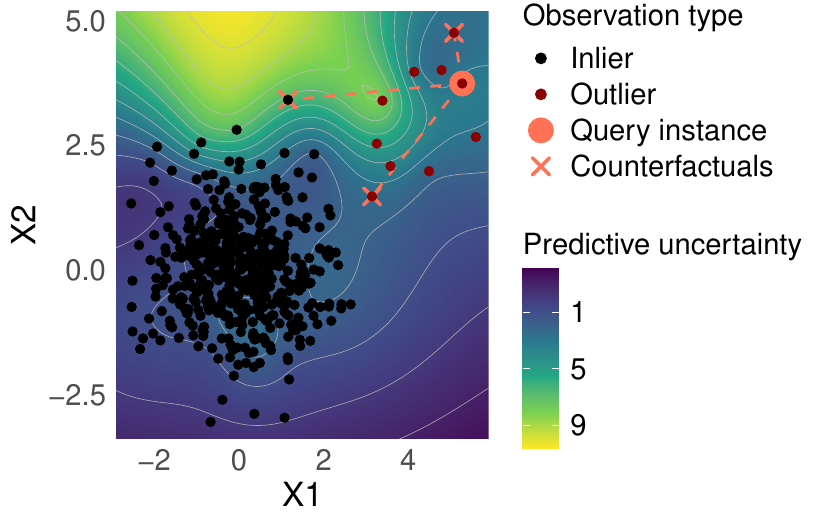}\includegraphics[width=.5\textwidth]{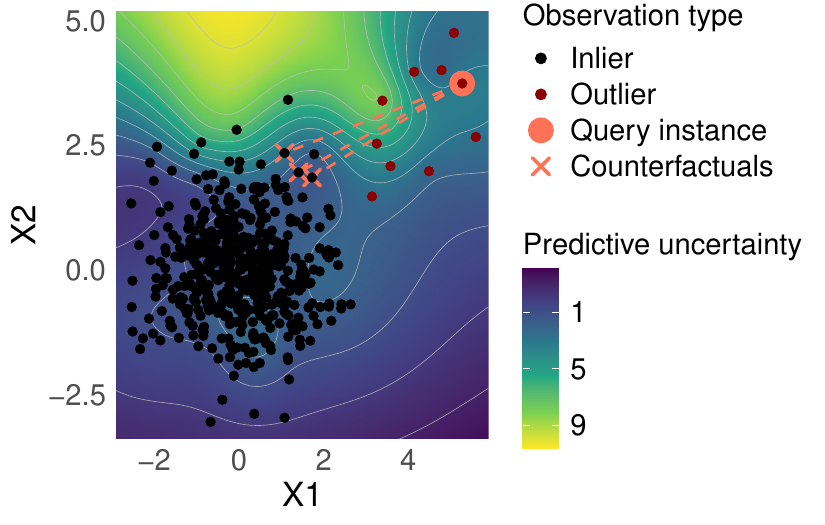}
    \caption{Counterfactual explanations generated by the WhatIf technique for a high-risk (outlier) query instance, without (left) and with (right) the reliability constraint formalized in Section~\ref{subsec:cp-constraint}. The background shading represents estimated predictive uncertainty. Without the constraint, recommended counterfactuals remain in the same high-uncertainty region as the query instance, purely because that region minimizes proximity. With the constraint, recommendations are relocated toward regions in which the model's own confidence is higher.}
    \label{fig:unique-challenge}
\end{figure}

In Figure~\ref{fig:unique-challenge}, the unconstrained counterfactuals (left panel) are valid, meaning their predictions fall within the desired range, but remain embedded in the same sparse, high-uncertainty region as the query instance, since nothing in the ordinary CE objective penalizes this. The constrained counterfactuals (right panel) are relocated toward denser, lower-uncertainty regions of the feature space, at the cost of being less proximal to the query instance. Because the underlying DGP is known in this synthetic setting, we can further confirm that the constrained recommendations correspond to an actual reduction in the true conditional claim rate $\mu(X_1,X_2)$, and not only in the fitted model's point prediction. This is the distinction between validity and reliability introduced in Section~\ref{subsec:reliability-notions}.

The gap is not specific to this example. The following result shows that a procedure can satisfy validity and actionability and still fail reliability.

\begin{proposition}[Validity and actionability do not imply reliability]
\label{prop:gap}
There exist a model $h$, a population $(\boldsymbol{X}, Y)$, and a CE procedure that satisfies Axioms~\ref{ax:validity} and~\ref{ax:actionability} and for which $\mathbb{P}(\boldsymbol{X} \in \mathcal{A}(\mathcal{P}),\, Y > y^u) = 1/4$, so that Axiom~\ref{ax:reliability} fails for every $\alpha < 1/4$.
\end{proposition}

\begin{proof}
Let $d=1$ with no immutable feature, $\boldsymbol{X}$ uniform on $\{0,1,2,3\}$, $Y = 2$ almost surely, and $y^u = 1$. Let $h(0)=h(1)=2$ and $h(2)=h(3)=1/2$, as for a model that underestimates the outcome in a sparsely observed region. The procedure that returns $x'=2$ for every query satisfies Axiom~\ref{ax:validity}, since $h(2)=1/2 \in [0, 1]$, and Axiom~\ref{ax:actionability} trivially. Its range is $\mathcal{A}(\mathcal{P})=\{2\}$, so $\mathbb{P}(\boldsymbol{X} \in \mathcal{A}(\mathcal{P}),\, Y > y^u) = \mathbb{P}(\boldsymbol{X}=2) = 1/4$.
\end{proof}

Reliability also cannot be certified from the model alone. If a procedure never uses outcome data, the actual outcome could exceed the target throughout the region that the model deems acceptable, and the procedure would fail Axiom~\ref{ax:reliability}. A procedure that satisfies Axiom~\ref{ax:reliability} must therefore use outcome data, which is what the calibration set in CP provides.

Closing this gap does not require a new uncertainty-quantification method. CP already provides a way to construct a prediction set $\widehat{C}_{1-\alpha}(\cdot)$ with a finite-sample marginal coverage guarantee, without distributional assumptions \citep{vovk2022algorithmic}. Unlike model-specific alternatives such as Bayesian posteriors or ensembles, it applies to any fitted model. What existing CE techniques lack is a way to apply that guarantee to the recommendation itself for regression tasks. We incorporate CP as a hard (or soft) constraint on the recommendations, imposed through the candidate set for WhatIf and as a post-hoc filter for the other techniques, detailed formally in Section~\ref{subsec:cp-constraint}.

\subsection{Reliable Recourse via a Conformal Prediction Constraint}
\label{subsec:cp-constraint}

\begin{definition}[Uncertainty-Aware Counterfactual Explanations for Regression]
A set of $M$ uncertainty-aware Counterfactual Explanations for insurance frequency or severity models (regression tasks) is obtained by solving
\begin{equation}
\label{eq:cp-ce}
\begin{gathered}
        \underset{\boldsymbol{x}'_1,\ldots,\boldsymbol{x}'_M \in \mathbb{R}^d}{\arg\min} \:
        \frac{1}{M} \sum_{m=1}^M
        l\left(h(\boldsymbol{x}'_m),Y'\right)
        +
        \frac{C_1}{M}\sum_{m=1}^M
        \theta(\boldsymbol{x}'_m,\boldsymbol{x}) \\
        \text{s.t.} \quad
        \forall m,\;
        \sup\left(\widehat{C}_{1-\alpha}(\boldsymbol{x}'_m)\right)
        \leq y^u,
\end{gathered}
\end{equation}
where $y^u$ is the upper bound of the desired prediction range $Y'$ and $\widehat{C}_{1-\alpha}(\cdot)$ denotes the CP prediction interval at confidence level $1-\alpha$, constructed from a held-out calibration set $\mathcal{D}_{\text{cal}}$. We treat $y^u$ as exogenous.
\end{definition}

\begin{proposition}[Reliability Bound]
\label{prop:coverage}
Let $\mathcal{A}_\alpha := \{\boldsymbol{x} \in \mathbb{R}^d : \sup(\widehat{C}_{1-\alpha}(\boldsymbol{x})) \leq y^u\}$ denote the set of profiles that satisfy the constraint in Equation~\ref{eq:cp-ce}, so that every counterfactual instance obtained by solving it lies in $\mathcal{A}_\alpha$. If $(\boldsymbol{X}, Y)$ is exchangeable with $\mathcal{D}_{\text{cal}}$, then
$$
\mathbb{P}\left(\boldsymbol{X} \in \mathcal{A}_\alpha,\; Y > y^u\right) \leq \alpha,
$$
and, whenever $\mathbb{P}(\boldsymbol{X} \in \mathcal{A}_\alpha) > 0$,
$$
\mathbb{P}\left(Y \leq y^u \mid \boldsymbol{X} \in \mathcal{A}_\alpha\right) \; \geq \; 1 - \frac{\alpha}{\mathbb{P}(\boldsymbol{X} \in \mathcal{A}_\alpha)},
$$
where the probabilities are taken over both $\mathcal{D}_{\text{cal}}$ and $(\boldsymbol{X}, Y)$.
\end{proposition}

Exchangeability means, roughly, that the new counterfactual policyholder is drawn from the same population as the calibration data. The proof, together with the supporting definitions of CP and exchangeability, is given in Appendix~\ref{sec:appendix-cp}. Intuitively, if a counterfactual policyholder profile is accepted by the constraint yet has an outcome above the target, that outcome lies above the upper end of an interval that was supposed to cover it. The miscoverage level $\alpha \in (0,1)$ of Split Conformal Prediction (SCP), which is the CP procedure used in this study \citep{zaffran2024uncertainty}, therefore bounds the probability of this event.

Proposition~\ref{prop:coverage} formalizes the sense in which the CP-constrained CE search is reliable, and not only valid. The constraint limits, at level $\alpha$, the rate at which counterfactual policyholder profiles are accepted as reliable yet have an actual outcome above the target. Ordinary CE validity concerns only the model's point prediction $h(\boldsymbol{x}'_m)$ and offers no statement about the actual outcome.

One limitation qualifies how the bound should be interpreted. A counterfactual is often not randomly drawn from the data distribution but is selected through the CE optimization, for example by targeting instances with low predicted frequency \citep{bilkhoo2025confexuncertaintyawarecounterfactualexplanations}. Exchangeability with the calibration set therefore cannot always be assumed for the returned counterfactual itself, and the bound concerns profiles accepted by the constraint, not the particular counterfactual returned.
To address this limitation, we use Locally Adaptive Conformal Prediction (LACP) \citep{lei2018distribution}, whose intervals adapt to local dispersion and thereby approximate conditional coverage, although they do not guarantee it (Section~\ref{sec:limitations}). To our knowledge, \citet{graziadei2025} is the only prior application of CP to claim frequency and severity models, and it does not consider counterfactual search. We are not aware of other work on CP-constrained CE for regression, including insurance pricing models.

Because the hard constraint in Equation~\ref{eq:cp-ce} can be infeasible at high confidence levels in high-uncertainty domains, we also allow a soft variant that selects, among generated candidates, those with the lowest CP upper bound. Proposition~\ref{prop:coverage} applies only to the hard constraint, and not to this soft variant. Section~\ref{sec:demonstration} reports an application in which the soft variant is required.

\begin{theorem}[Sufficiency of the CP constraint]
\label{thm:sufficiency}
Let $\mathcal{P}$ be a CE procedure that returns only instances $\boldsymbol{x}'$ with $x'_j = x_j$ for all $j \in \mathcal{I}$ and $\sup(\widehat{C}_{1-\alpha}(\boldsymbol{x}')) \leq y^u$, where $\widehat{C}_{1-\alpha}$ is the locally adaptive interval of Equation~(\ref{eq:appendix-lacp}), built from a calibration set $\mathcal{D}_{\text{cal}}$ exchangeable with $(\boldsymbol{X}, Y)$. Then $\mathcal{P}$ satisfies Axioms~\ref{ax:validity}, \ref{ax:actionability}, and~\ref{ax:reliability}.
\end{theorem}

\begin{proof}
Axiom~\ref{ax:actionability} holds by the restriction on the returned instances. For Axiom~\ref{ax:validity}, the interval $\widehat{C}_{1-\alpha}(\boldsymbol{x}')$ is centered on $h(\boldsymbol{x}')$, so $h(\boldsymbol{x}') \leq \sup(\widehat{C}_{1-\alpha}(\boldsymbol{x}')) \leq y^u$, and $h(\boldsymbol{x}') \geq 0$ because the model predicts non-negative outcomes. For Axiom~\ref{ax:reliability}, $\mathcal{A}(\mathcal{P}) \subseteq \mathcal{A}_\alpha$, so $\mathbb{P}(\boldsymbol{X} \in \mathcal{A}(\mathcal{P}),\, Y > y^u) \leq \mathbb{P}(\boldsymbol{X} \in \mathcal{A}_\alpha,\, Y > y^u) \leq \alpha$ by Proposition~\ref{prop:coverage}.
\end{proof}

All four techniques with the CP filter of Table~\ref{tab:cp-incorporation} satisfy the conditions of Theorem~\ref{thm:sufficiency}.

Reliability is not free. Restricting the search to low-uncertainty regions can only make the underlying proximity objective weakly harder to satisfy, and we formalize this trade-off and its dependence on the confidence level.


\begin{proposition}[Cost of Reliability]
\label{prop:cost}
Let
$$
v^\star_{\text{unc}} := \min_{\boldsymbol{x}'_1,\ldots,\boldsymbol{x}'_M \in \mathbb{R}^d} \; \frac{1}{M}\sum_{m=1}^M l\left(h(\boldsymbol{x}'_m),Y'\right) + \frac{C_1}{M}\sum_{m=1}^M \theta(\boldsymbol{x}'_m,\boldsymbol{x})
$$
denote the optimal value of the unconstrained CE problem (the $M$-instance form of Equation~\ref{eq:min-ce}), and let $v^\star_{\text{CP}}(\alpha)$ denote the optimal value of the CP-constrained problem in Equation~\ref{eq:cp-ce} with feasible set $\mathcal{F}_\alpha \subseteq (\mathbb{R}^d)^M$, which additionally requires $\sup(\widehat{C}_{1-\alpha}(\boldsymbol{x}'_m)) \leq y^u$ for all $m \in \{1,\ldots,M\}$, with $v^\star_{\text{CP}}(\alpha) = +\infty$ if $\mathcal{F}_\alpha = \emptyset$. For a fixed calibration set and $\alpha' \leq \alpha$,
$$
v^\star_{\text{unc}} \; \leq \; v^\star_{\text{CP}}(\alpha) \; \leq \; v^\star_{\text{CP}}(\alpha'),
$$
and the first inequality is an equality if and only if some unconstrained minimizer already satisfies the CP reliability constraint.
\end{proposition}

The proof is in Appendix~\ref{sec:ec-cost-proof}. Proposition~\ref{prop:cost} says that whenever an unconstrained recommendation already happens to satisfy the reliability constraint, meaning the model is already confident in that region, reliability is obtained at no additional proximity cost. Because the CP interval is widest in sparse, high-uncertainty regions, the constraint is therefore most likely to bind, and impose a cost, for the high-risk, sparsely observed query instances motivating this paper (Section~\ref{subsec:unique-challenges}), and to be inert where the model is already confident. The monotonicity in $\alpha$ also implies a cost--reliability frontier, which Table~\ref{tab:reliability_comparison} illustrates for the telematics data. When both problems attain validity ($l=0$), the inequality in Proposition~\ref{prop:cost} is a weakly larger proximity cost. This is consistent with the uncertainty-aware recommendations in Table~\ref{tab:flood-query-instance} changing at least as many features as their unconstrained counterparts, and more in most cases.

\subsection{Incorporating the Constraint into the Four CE Techniques}
\label{subsec:cp-incorporation}

For WhatIf, the candidate set is redefined to retain only training instances whose CP upper bound is at most $y^u$. For the three remaining CE techniques, NICE, MOC, and DiCE, candidates are generated without this constraint, and a post-hoc filter then retains those that satisfy it. As future work, NICE could be extended so that the initial candidate is drawn from the redefined reliable set used by WhatIf. For both MOC and DiCE, incorporating the CP constraint as a fifth objective within the genetic algorithm is a possible extension. Table~\ref{tab:cp-incorporation} summarizes these mechanisms.

\begin{table}[!htbp]
\centering
\small
\begin{tabularx}{\textwidth}{l X}
\toprule
Method & How the CP constraint is incorporated\\
\midrule
WhatIf & Candidate set $\mathcal{D}_{\text{train}}'$ additionally requires $\sup(\widehat{C}_{1-\alpha}(\boldsymbol{x}')) \leq y^u$\\
NICE & Post-hoc filter over generated candidates\\
MOC & Post-hoc filter over generated candidates\\
DiCE & Post-hoc filter over generated candidates\\
\bottomrule
\end{tabularx}
\caption{How the CP reliability constraint (Axiom~\ref{ax:reliability}) of Equation~\ref{eq:cp-ce} is incorporated into each of the four CE techniques of Table~\ref{tab:ce-properties}.}
\label{tab:cp-incorporation}
\end{table}


\section{Empirical Illustration}
\label{sec:demonstration}

This section applies the approach to two datasets, one real and one synthetic, and compares the four CE techniques with and without the CP constraint.

The two datasets play different roles. The flood data are real, the outcome is a severity, and predictive uncertainty is high, so only the soft constraint is feasible. The telematics data are synthetic, the outcome is a frequency, and predictive uncertainty is lower, so the hard constraint can be used and the four techniques can be compared across a test sample. For all four retained CE techniques (WhatIf, NICE, MOC and DiCE), the hyperparameter values, their justifications, and comments on tuning for other datasets are reported in Appendix~\ref{sec:ce-parameters}. Common parameters are matched across methods whenever possible, particularly for the MOC and DiCE genetic algorithms. In the following, each dataset is randomly split into a training set (60\%), a calibration set (30\%), and a test set (10\%), from which query instances are selected.







\subsection{Flood Insurance}
\label{subsec:demo-flood}


We apply the four CE techniques, with and without the CP constraint of Equation~\ref{eq:cp-ce}, to severity predictions from a gradient-boosted tree model (XGBoost, $h$: the RMSE is 0.195 on $\mathcal{D}_{\text{test}}$). The model is fitted on FEMA's \emph{FIMA NFIP Redacted Claims} dataset\footnote{Federal Emergency Management Agency (FEMA), OpenFEMA Dataset: \textit{FIMA NFIP Redacted Claims - v2}. Retrieved from https://www.fema.gov/openfema-data-page/fima-nfip-redacted-claims-v2 on July 13, 2026, at 11:19 PM EST. This product uses the FEMA OpenFEMA API but is not endorsed by FEMA. The Federal Government and FEMA cannot vouch for the data or analyses derived from these data after they have been retrieved from the Agency's website(s).}. We restrict the data to $n=57{,}832$ Florida flood claims from three events between 2023 and 2024, for which reported amounts are sufficiently mature and unlikely to still be developing (data access and full variable descriptions are provided in Appendix~\ref{sec:appendix-flood-data}). By construction, these records correspond exclusively to claims with positive losses, with no non-claiming policies represented. The outcome is the loss ratio (LR), the damage amount over insured value, conditional on a claim. We refer to the model as a severity model and report severity in loss-ratio units.

Policy duration is unavailable because the dataset has no unique policy identifier, so we assume an exposure of one year for each policy and year. We also assume that 100\% of each property is insured. Predicted severity could be multiplied by the coverage percentage to obtain real-world severity values.

Of the eleven building and location characteristics retained as predictors, three are actionable. They are whether the building is elevated (\emph{elevation}), the type of basement or enclosure (\emph{basement}), and the location of insured contents relative to the lowest floor (\emph{contents\_loc}). Of the remaining eight predictors, seven (number of floors, non-profit status, occupancy, primary residence, rental status, building age, and flood zone; see Table~\ref{tab:variables-flood} in Appendix~\ref{sec:appendix-flood-data}) are structural or policy-application characteristics available at underwriting, but the cause of damage is recorded only once a loss has occurred and so is unavailable to a policyholder before that loss.

Recommended changes to elevation, basement configuration, or contents location are evaluated holding the cause of damage fixed at its observed value. If the cause of damage is itself downstream of these features, for instance because an elevated building is less likely to sustain foundation-related damage, this treatment could over- or understate the actual effect on severity. This case study applies the CP-constrained CE approach to claim severity prediction. 

For two query policyholders, $\boldsymbol{x}_1$ and $\boldsymbol{x}_2$, with predicted LR of 0.46 
and 0.40, 
Table~\ref{tab:flood-query-instance} reports four example counterfactual recommendations for each, two without and two incorporating the CP constraint. While all four CE techniques were applied to both $\boldsymbol{x}_1$ and $\boldsymbol{x}_2$, the results presented here were selected from the WhatIf technique as illustrative examples (see the reproducibility code for further results from the other CE techniques). The predicted value for $\boldsymbol{x}_1$ (respectively, $\boldsymbol{x}_2$) lies approximately at the 95th percentile (respectively, the 68th percentile) of the empirical distribution of predicted values on the test set; while their true LR values lie approximately at the 80th and 96th percentiles of the empirical LR distribution. For both $\boldsymbol{x}_1$  and $\boldsymbol{x}_2$, the associated predictive uncertainty which we quantify using the LACP interval width corresponds to the 76th percentile of the interval widths within the test set. Thus, we characterize both $\boldsymbol{x}_1$ and $\boldsymbol{x}_2$ as high-risk, high-uncertainty policyholders. We aim to reduce their predicted severity by a fixed amount of 10\%, a target chosen for illustration, corresponding to a value of $y^u$ of approximately 0.41 
and 0.36, respectively.

\begin{table}[!htbp]
\centering
\small
\setlength{\tabcolsep}{4pt}
\begin{tabular}{cccccccc}
\toprule
Query & CE & \textit{basement} & \textit{elevation} & \textit{contents\_loc} &
Pred. & CP constr. & CP bound \\
\midrule
\multirow{4}{*}{$\boldsymbol{x}_1$}
& $\boldsymbol{x}'_{1,1}$ & 0 & \textbf{\textit{yes}} & \textit{Lowest\_floor\_only} & 0.38 & $\times$ & 0.64 \\
& $\boldsymbol{x}'_{1,2}$ & 0 & \textit{no} & \textbf{\textit{Lowest\_floor\_and\_above}} & 0.31 & $\times$ & 0.56 \\
& $\boldsymbol{x}'_{1,3}$ & 0 & \textbf{\textit{yes}} & \textbf{\textit{Lowest\_floor\_and\_above}} & 0.27 &\checkmark & 0.52 \\
& $\boldsymbol{x}'_{1,4}$ & \textbf{1} & \textbf{\textit{yes}} & \textit{Lowest\_floor\_only}& 0.21 &  \checkmark & 0.45 \\
\midrule
\multirow{4}{*}{$\boldsymbol{x}_2$}
& $\boldsymbol{x}'_{2,1}$ & \textbf{0} & \textbf{\textit{yes}} & \textit{Lowest\_floor\_only} & 0.34 & $\times$ & 0.58 \\
& $\boldsymbol{x}'_{2,2}$ & 3 & \textbf{\textit{yes}} & \textit{Lowest\_floor\_only} & 0.27 & $\times$ & 0.51 \\
& $\boldsymbol{x}'_{2,3}$ & \textbf{1} & \textit{no} & \textbf{\textit{Lowest\_floor\_and\_above}} & 0.23 &\checkmark & 0.45 \\
& $\boldsymbol{x}'_{2,4}$ & \textbf{1} & \textbf{\textit{yes}} & \textit{Lowest\_floor\_only}& 0.17 &  \checkmark & 0.38 \\
\bottomrule
\end{tabular}
\caption{Actionable, uncertainty-aware Counterfactual Explanations for flood severity reduction, for two query policyholders $\boldsymbol{x}_1$ and $\boldsymbol{x}_2$. Changed actionable feature values are shown in bold. ``Pred.'' is the predicted loss ratio. ``CP constr.'' indicates whether the soft CP reliability constraint is considered when selecting the counterfactual instance. ``CP bound'' is the 80\%-confidence upper bound of predicted severity for the recommended counterfactual.}
\label{tab:flood-query-instance}
\end{table}

For $\boldsymbol{x}_1$ (no basement, not elevated, contents on the lowest floor only), moving part of the contents to the upper floors alone ($\boldsymbol{x}'_{1,2}$) is recommended when reliability is not considered among the criteria for the proposed counterfactuals. Once the soft reliability constraint is enforced, the search instead recommends a more conservative, two-feature change ($\boldsymbol{x}'_{1,3}$ or $\boldsymbol{x}'_{1,4}$), which yields both a lower point prediction and a substantially lower CP upper bound (from a LR of 0.56 down to 0.52--0.45). This pattern, uncertainty-aware recommendations requiring equal or more changes, but delivering more credible reductions, recurs for $\boldsymbol{x}_2$, illustrated in Figure~\ref{fig:second-query-ex-flood}. It is consistent with the mechanism behind Proposition~\ref{prop:cost}. 

Because predictive uncertainty for flood severity modeling is high, the hard CP constraint was infeasible even at an 80\% confidence level. Consistent with Section~\ref{subsec:cp-constraint}, we therefore report the soft-constraint variant. It does not enforce the target on the CP upper bound, so Proposition~\ref{prop:coverage} does not certify these recommendations, and their upper bounds should be read as relative indications of reliability. The reliability constraint trades a looser proximity objective for a recommendation whose predicted LR reduction is backed by a lower CP upper bound, rather than one whose apparent benefit is an artifact of the fitted model's own uncertainty.

The miscoverage level $\alpha \in (0,1)$ is selected by the modeler according to the desired confidence level, with common choices for CP intervals ranging from approximately 1\% to 20\% \citep{idrissi2025unveilsourcesuncertaintyfeature, lei2018distribution, bilkhoo2025confexuncertaintyawarecounterfactualexplanations, adams2025individualisedcounterfactualexamplesusing}. For the flood case study in Table~\ref{tab:flood-query-instance}, we report the 80\% confidence upper bound, corresponding to the lower end of the $(1-\alpha)$-range in related works, to illustrate the high uncertainty of the LR model.



\begin{figure}[!htbp]
    \centering
    \includegraphics[width=.8\textwidth]{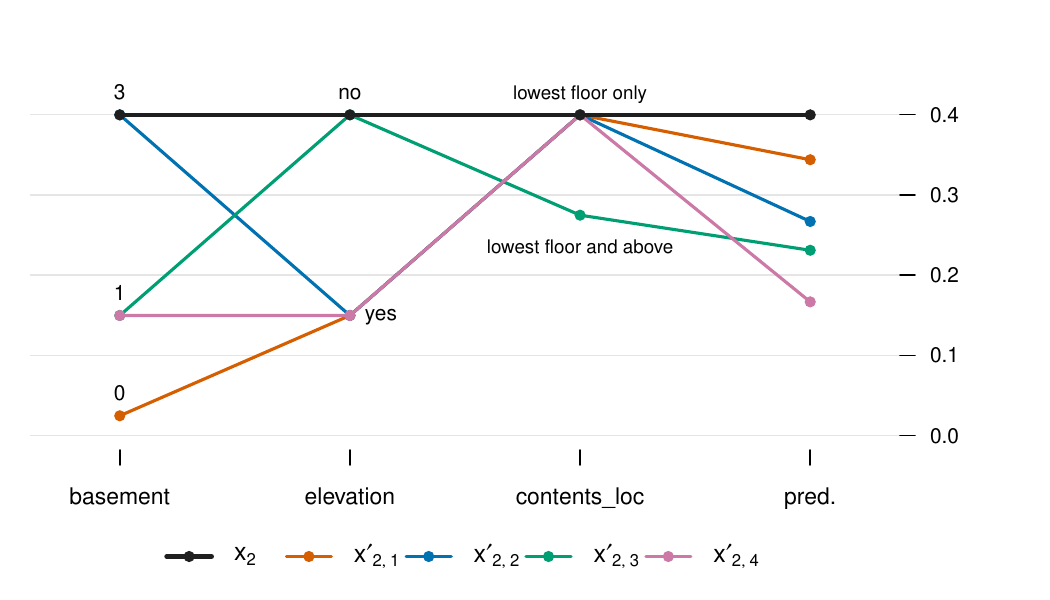}
    \caption{Decision plot for policyholder $\boldsymbol{x}_2$. The plot shows the recommended changes in actionable features to reach the four counterfactual instances ($\boldsymbol{x}'_{2,1}$, $\boldsymbol{x}'_{2,2}$, $\boldsymbol{x}'_{2,3}$, and $\boldsymbol{x}'_{2,4}$ of Table~\ref{tab:flood-query-instance}), along with their corresponding predicted flood severity. Each line starts from the query values of $\boldsymbol{x}_2$ (black) and follows one counterfactual across the three actionable features to its predicted value (right axis).}
    \label{fig:second-query-ex-flood}
\end{figure}

Policyholder $\boldsymbol{x}_2$ owns a property with a subgrade crawlspace (\emph{basement}: 3), is not classified as an elevated building, and has contents located on the lowest floor only, with a predicted LR of 0.40 and an associated flood severity of \$169{,}657. Figure~\ref{fig:second-query-ex-flood} traces four example recommended paths to a target LR below 0.36, using the four CE techniques. Without the reliability constraint, elevating the building ($\boldsymbol{x}'_{2,2}$), alone or combined with removing the crawlspace ($\boldsymbol{x}'_{2,1}$), is sufficient. With the constraint, the search instead recommends renovating the crawlspace into a finished basement in addition to relocating contents ($\boldsymbol{x}'_{2,3}$), changes the model associates with materially lower predictive uncertainty as well as lower predicted severity.

Beyond the predicted risk reduction, the exercise also has diagnostic value for the insurer. Comparing $\boldsymbol{x}'_{1,1}$ and $\boldsymbol{x}'_{1,2}$ in Table~\ref{tab:flood-query-instance} shows that the fitted model attributes a smaller severity reduction to elevating the building than to relocating contents. This is difficult to reconcile with elevation's apparent importance for flood resilience, and suggests the model may be under-informed by the relatively small number of elevated buildings in the training data (14\%), an issue the CE search surfaces as a byproduct rather than by design.

In this application, the reliability constraint therefore does not simply trade a smaller benefit for more reliability, as Proposition~\ref{prop:cost} alone might suggest. It selects recommendations that are less proximal, more reliable, and associated with a larger predicted risk reduction. Whether the pattern holds generally or is specific to this flood dataset is an empirical question for a portfolio-scale study (Section~\ref{sec:limitations}). Section~\ref{subsec:demo-telematics} provides a first check in a second domain, with a frequency rather than a severity model.

\subsection{Driver Telematics}
\label{subsec:demo-telematics}

Usage-based automobile insurance uses telematics data on driving behavior in claim frequency models. To illustrate the approach on a frequency model, we apply the four CE techniques from Table~\ref{tab:ce-properties} to the synthetic driver-telematics dataset of \citet{so2021syntheticdatasetgenerationdriver}, which reproduces the statistical properties of a real Canadian usage-based insurance program (2013--2016, approximately 100,000 policies) without disclosing individual policyholders' data. After preprocessing, the dataset comprises $n=98{,}058$ policyholders and $d=9$ explanatory features, three of which are actionable (data access and variable descriptions are provided in Appendix~\ref{sec:appendix-telematics}). \textit{Pct.drive.rush} is the driver's predominant rush-hour driving pattern (no rush-hour driving, mostly morning, mostly evening, or both). \textit{Accel.14miles} indicates whether the driver had at least one extreme acceleration event exceeding 14 mph/s per 1{,}000 miles, and \textit{Brake.14miles} is the equivalent indicator for extreme braking. These actionable features are obtained by slightly modifying the original features so that changes to them are actionable for policyholders and easier to implement. We fit a Poisson GLM ($h$: the Poisson deviance is 0.276 on $\mathcal{D}_{\text{test}}$) to predict annual claim frequency.

For a query policyholder $\boldsymbol{x}$ with a predominantly evening rush-hour driving pattern and at least one extreme acceleration and braking event, the predicted claim frequency is 22\%, with an 80\%-confidence CP upper bound of 39\%. The predicted value of the query instance $\boldsymbol{x}$ lies at the 99.9th percentile of the empirical distribution of predicted values on the test set; while its actual frequency lies at the 99.7th percentile of the empirical frequency distribution. The associated predictive uncertainty, which we quantify using the LACP interval width, corresponds to the 99.7th percentile of the interval widths in the test set. Table~\ref{tab:telematics-query-instance} reports four examples of counterfactual recommendations, with the aim of reducing the predicted frequency by 10\%, a target chosen for illustration, thereby targeting a predicted frequency below 20\%. While all four CE techniques were applied to $\boldsymbol{x}$, the results presented here were selected from the WhatIf and NICE techniques as illustrative examples (see the reproducibility code for further results from the other CE techniques). Unlike the flood application, predictive uncertainty of $h$ here is low enough that the hard CP constraint of Equation~\ref{eq:cp-ce} is feasible at the 80\% confidence level. 
We select $\alpha = 20\%$ because it allows us to obtain a sufficient number of counterfactuals across the different CE techniques while satisfying the CP constraint. Choosing a higher confidence level for the CP intervals would mean that $\boldsymbol{x}'_3$ in Table~\ref{tab:telematics-query-instance} could no longer be considered a reliable CE. For the selection of $\alpha \in (0,1)$, we recommend starting with the hard constraint at the lowest confidence level considered in related works, 80\%, and decreasing $\alpha$ gradually until no counterfactual instance is reliable, a minimum number of instances is reached, or $\alpha = 1\%$.

\begin{table}[!htbp]
\centering
\small
\begin{tabular}{ccccccc}
\toprule
CE & \textit{Pct.drive.rush} & \textit{Accel.14miles} & \textit{Brake.14miles} & Pred. & CP constr. & CP bound \\
\midrule
$\boldsymbol{x}'_{1}$ & mostly\_pm & \textbf{no} & yes & 18\% & $\times$ & 35\% \\
$\boldsymbol{x}'_{2}$ & mostly\_pm & \textbf{no} & \textbf{no} & 12\% & $\times$ & 24\% \\
$\boldsymbol{x}'_{3}$ & \textbf{none} & yes & \textbf{no} & 10\% & \checkmark & 20\% \\
$\boldsymbol{x}'_{4}$ & \textbf{mostly\_am} & \textbf{no} & \textbf{no} & 7\% & \checkmark & 18\% \\
\bottomrule
\end{tabular}
\caption{Actionable, uncertainty-aware Counterfactual Explanations for claim frequency reduction, for a query policyholder with predicted frequency 22\% and 80\%-confidence CP upper bound 39\%. Changed actionable feature values are shown in bold. ``Pred.'' is the predicted claim frequency. ``CP constr.'' indicates whether the CP reliability constraint is applied. ``CP bound'' is the 80\%-confidence upper bound of predicted frequency for the recommended counterfactual.}
\label{tab:telematics-query-instance}
\end{table}

\begin{figure}[!htbp]
    \centering
    \includegraphics[width=.8\textwidth]{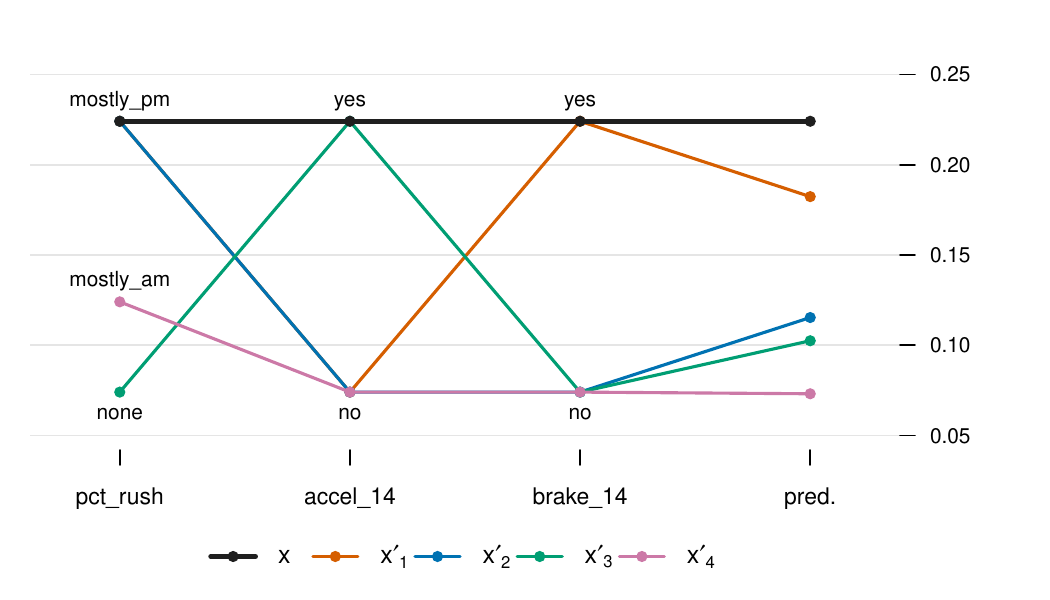}
    \caption{Decision plot for a query policyholder in the telematics dataset. The plot shows the recommended changes in actionable features to reach the four counterfactual instances of Table~\ref{tab:telematics-query-instance}, along with their corresponding predicted claim frequency. Each line starts from the query values (black) and follows one counterfactual across the three actionable features to its predicted value (right axis).}
    \label{fig:telematics-decision-plot}
\end{figure}

Figure~\ref{fig:telematics-decision-plot} traces the four recommended paths to obtain CEs from Table~\ref{tab:telematics-query-instance}. Without the reliability constraint, changing only the acceleration indicator ($\boldsymbol{x}'_1$) reduces the predicted frequency to 18\%, and additionally changing the braking indicator ($\boldsymbol{x}'_2$) reduces it further to 12\%. Neither recommendation is reliable at the 80\% confidence level. Their CP upper bounds, 35\% and 24\%, both remain above the 20\% target, so the reduction in the point prediction is not backed by a corresponding reduction in the model's own uncertainty. With the constraint, $\boldsymbol{x}'_3$ instead recommends eliminating rush-hour driving entirely together with the braking change, reaching a predicted frequency of 10\% with a CP upper bound of exactly 20\%. The recommendation $\boldsymbol{x}'_4$ shifts driving to the morning rush hour and eliminates both extreme acceleration and braking, reaching 7\% with an 18\% upper bound.

As in the flood application, the reliability constraint here does not simply trade predicted benefit for reliability. The recommendations $\boldsymbol{x}'_3$ and $\boldsymbol{x}'_4$ achieve larger reductions in the point prediction as well as reliable ones, a 54.5\% and 68.2\% reduction from the 22\% baseline, compared with 18.2\% and 45.5\% for the two unconstrained recommendations.

The pattern is not uniform across the number of features changed. The recommendation $\boldsymbol{x}'_2$ (unconstrained) already changes two features, the same as $\boldsymbol{x}'_3$ (constrained), so reliability is not purchased with strictly more changes in every pairwise comparison, only on average (see Table~\ref{tab:telematics-quantified}). This is consistent with Proposition~\ref{prop:cost}, which predicts a cost only where the unconstrained optimum falls outside the low-uncertainty region.


\begin{table}[!htbp]
\centering
\small
\begin{tabular}{ccccc}
\toprule
Query & CE & CP constr. & Features changed & Frequency reduction \\
\midrule
\multirow{4}{*}{$\boldsymbol{x}$}
& $\boldsymbol{x}'_{1}$ & $\times$ & 1.00 & 18.2\% \\
& $\boldsymbol{x}'_{2}$ & $\times$ & 2.00 & 45.5\% \\
& $\boldsymbol{x}'_{3}$ & \checkmark & 2.00 & 54.5\% \\
& $\boldsymbol{x}'_{4}$ & \checkmark & 3.00 & 68.2\% \\
\midrule
\multirow{8}{*}{$\mathcal{D}_{\text{test}}$}
& \multirow{2}{*}{WhatIf} 
& $\times$ & 1.43 & 26.2\% \\
& & \checkmark & 2.73 & 61.6\%\\
\cline{2-5}

& \multirow{2}{*}{NICE} 
& $\times$ & 1.16 & 28.2\%\\
& & \checkmark & 2.59 & 62.0\% \\
\cline{2-5}

& \multirow{2}{*}{MOC} 
& $\times$ & 1.09 & 21.8\% \\
& & \checkmark & 3.00 & 59.7\%\\
\cline{2-5}

& \multirow{2}{*}{DiCE} 
& $\times$ & 1.57 & 25.5\% \\
& & \checkmark & 2.68 & 61.1\% \\
\midrule

\multicolumn{2}{c}{\textbf{Average}} & $\times$ & \textbf{1.40} & \textbf{25.7\%} \\
\cline{1-5}
\multicolumn{2}{c}{\textbf{Average}} & \checkmark & \textbf{2.69} & \textbf{61.4\%}\\
\bottomrule
\end{tabular}
\caption{Predicted frequency reduction relative to the query instance's own prediction, by whether the CP reliability constraint (``CP constr.'') is applied, illustrating the proximity-cost pattern characterized by Proposition~\ref{prop:cost}. Rows related to $\boldsymbol{x}$ show the changes from Table~\ref{tab:telematics-query-instance}, while the others show instance-weighted averages over $\mathcal{D}_{\text{test}}$. The bold rows give the overall averages.
}
\label{tab:telematics-quantified}
\end{table}

\subsubsection{Global Comparison and Practical Recommendations}
\label{subsubsec:comparison-telematics}

Table~\ref{tab:telematics-global-results} summarizes the global performance of the four CE methods on 1,000 test samples randomly drawn from the telematics dataset. It reports validity and reliability (Axioms~\ref{ax:validity} and~\ref{ax:reliability}), the desirable properties of plausibility, proximity, sparsity, and multiplicity, and computational and overlap metrics. Actionability (Axiom~\ref{ax:actionability}) holds for all four methods and is not reported.

\begin{table}[!htbp]
\centering
\small
\begin{tabular}{lcccccccccc}
\toprule
& \multicolumn{2}{c}{\textbf{Axioms}} & \multicolumn{4}{c}{\textbf{Desirable properties}} & \multicolumn{3}{c}{\textbf{Computation and overlap}}\\
\cmidrule(lr){2-3}\cmidrule(lr){4-7}\cmidrule(lr){8-10}
Method & Valid. & Reliab. & Plausib. & Prox. & Sparsity & $M$ & Time & Shared & Dup. \\
\midrule
WhatIf & \textbf{100\%} & 3.34\% & 0.008 & 0.164 & 1.48 & 4.34 & \textbf{0.04} & 48\% & \textbf{1} \\
\midrule
NICE & \textbf{100\%} & \textbf{3.78\%} & 0.008 & 0.135 & 1.22 & 1.42 & 2.46 & 94\% & \textbf{1} \\
\midrule
MOC & 74\% & 0.23\% & \textbf{0.003} & \textbf{0.122} & \textbf{1.09} & 1.37 & 0.71 & 74\% & \textbf{1} \\
\midrule
DiCE & \textbf{100\%} & 2.98\% & 0.013 & 0.177 & 1.60 & \textbf{5.15} & 2.32 & 0\% & 6.34 \\
\bottomrule
\end{tabular}
\caption{Global evaluation of Counterfactual Explanation (CE) methods on 1,000 samples from the telematics test set. Valid. and Reliab. correspond to Axioms~\ref{ax:validity} and~\ref{ax:reliability}. Plausibility (``Plausib.'') is the Gower distance to the closest instance in $\mathcal{D}_{\text{train}}$. Proximity (``Prox.'') is the Gower distance to the query instance. Sparsity is the number of changed features. Validity (``Valid.'') is the percentage of valid CEs. Reliability is the percentage of reliable CEs at an 80\% confidence level. ``Time'' is the average execution time, in second, per query. ``Shared'' is the percentage of CEs found in the three other methods. ``Dup.'' is the average number of duplicate CEs per query, and $M$ is the average multiplicity per query.}
\label{tab:telematics-global-results}
\end{table}

Table~\ref{tab:telematics-global-results} shows clear differences between the methods. MOC provides the most plausible and closest counterfactuals and requires the fewest feature changes, but at the cost of a lower validity rate (74\%)\footnote{Validity can be added as a post-hoc filter directly in the R code of the package \texttt{counterfactuals} \citep{JSSv115i09} but may result in no counterfactuals.} and a very low reliability rate (0.23\%). NICE combines perfect validity with the highest reliability rate (3.78\%) and relatively low Gower distances and sparsity. WhatIf is by far the fastest method, at 0.04 per query, with perfect validity. DiCE produces the largest number of counterfactuals per query ($M=5.15$), but has the highest Gower distances and sparsity, a long execution time (2.32), and substantially more duplicates. Increasing the population size and the number of generations of the genetic algorithm in MOC worsens plausibility, proximity, sparsity, and validity (Appendix~\ref{sec:ce-parameters}).

The choice of method should therefore depend on the practical objective. WhatIf suits large-scale applications because of its speed and perfect validity. NICE offers the best balance of validity, reliability, proximity, and sparsity. MOC suits settings that prioritize plausibility and sparsity, especially with models of low predictive uncertainty. DiCE suits settings that need a larger and more diverse set of CEs, when computational cost is not a concern. Most CEs proposed by NICE are already identified by WhatIf and MOC, so NICE may be unnecessary once one of them has been applied. These results come from a single telematics dataset with three actionable features, and the relative performance of the methods may differ for other data, models, and larger numbers of actionable features. Appendix~\ref{sec:ec-telematics-notes} gives further notes on this comparison.


Table~\ref{tab:reliability_comparison} reports the proportion of reliable CEs and the associated costs when the hard constraint is imposed at $\alpha=0.20$ and $\alpha=0.05$. A stricter confidence level lowers the proportion of reliable CEs for every method, and NICE has the highest proportion at both levels. As $\alpha$ decreases, the proximity and sparsity costs generally increase, except for MOC, for which they remain constant. The additional reduction in predicted claim frequency resulting from the reliability constraint is substantial only for MOC, while remaining almost constant for the other CE techniques, thus not contradicting our previous claims about the effect of adding the reliability constraint.

\begin{table}[!htbp]
\centering
\small
\begin{tabular}{ccrrrr}
\toprule
Method & $\alpha$ & Proportion reliable
& $\Delta_{\text{proximity}}$
& $\Delta_{\text{sparsity}}$
& $\Delta_{\text{frequency}}$ \\
\midrule
\multirow{2}{*}{WhatIf}
& 0.20 & 3.34\% & 0.118    &      1.06       &   24.9\% \\
& 0.05 & 1.04\% & 0.135   &       1.22    &       24.6\% \\
\midrule
\multirow{2}{*}{NICE}
& 0.20 & 3.78\% & 0.151  &        1.36     &     27.0\%\\
& 0.05 & 1.08\% & 0.175    &      1.57      &    27.3\%\\
\midrule
\multirow{2}{*}{MOC}
& 0.20 & 0.23\% & 0.167 &          1.50     &     25.6\%\\
& 0.05 & 0.08\% & 0.167 &          1.50    &     29.2\%\\
\midrule
\multirow{2}{*}{DiCE}
& 0.20 & 2.98\% & 0.105 &        0.94    &      24.4\% \\
& 0.05 & 0.95\% & 0.126  &       1.14  &         24.3\% \\
\bottomrule
\end{tabular}
\caption{Illustration of the proximity and sparsity costs
($\Delta_{\text{proximity}}$, $\Delta_{\text{sparsity}}$) and the
frequency reduction ($\Delta_{\text{frequency}}$) resulting from the
addition of the hard reliability constraint, for different values of the Conformal Prediction interval coverage level $\alpha$, averaged over 1,000 randomly drawn samples from the telematics test set. $\Delta$ denote the average difference between constrained and unconstrained counterfactuals, and is it computed only over instances from with a reliable CE.}
\label{tab:reliability_comparison}
\end{table}

Across the two applications, three patterns stand out. First, recommendations that are valid in the model's own terms can be unreliable. In the flood data, the unconstrained recommendations have CP upper bounds above the target, and in the telematics data the CP upper bounds of the two unconstrained recommendations (35\% and 24\%) both exceed the 20\% target. Second, imposing the constraint selects recommendations with lower CP upper bounds and larger predicted risk reductions, at the price of changing more features on average (2.69 against 1.40 in the telematics test sample, Table~\ref{tab:telematics-quantified}). Third, the techniques differ in how readily they produce reliable recommendations, and a stricter confidence level reduces the proportion of reliable recommendations for every technique (Table~\ref{tab:reliability_comparison}). These patterns come from two datasets and a small number of query policyholders, and they should be read as illustrations.

\section{Discussion}
\label{sec:discussion}

We began by asking how an individual policyholder's algorithmically determined premium should be explained to them. Our results show that an explanation can be valid and actionable and still fail the policyholder, and that a Conformal Prediction constraint can screen out the recommendations that the model is too uncertain about. In our applications, only a small share of the valid recommendations is reliable, and the constrained recommendations are more reliable and are associated with larger predicted risk reductions, at the price of changing more features. We now discuss what these results imply for practice and research, and where their limits lie.

\subsection{Implications for Practice}

First, for insurers, the approach offers a way to explain premiums to individual policyholders in actionable terms without disclosing proprietary model internals, consistent with the original motivation for CE in \citet{wachter2018counterfactual}. It may also help insurers meet the GDPR's explanation-related obligations \citep{GDPR2016} and, for life and health insurers specifically, the EU AI Act's transparency obligations for high-risk pricing systems \citep{eu_ai_act_annex3}. Because CE recommendations are individually diagnostic, aggregating them across a portfolio can help insurers identify where a pricing model behaves inconsistently with domain expectations, as the elevation-related modeling issue in Section~\ref{sec:demonstration} illustrates. Recommendations can also support prevention and mitigation programs by quantifying the premium benefit of a specific mitigation action for a specific policyholder. In particular, for climate prevention, considering causal relationships among features further supports the recommended directions.

Second, for regulators, the three axioms offer a possible checklist for judging whether an insurer's premium explanations are usable and trustworthy, since an explanation can be valid and actionable and still fail the policyholder if it is not reliable. Properties such as plausibility and sparsity are not conditions for usable recourse, but regulators could ask insurers to report them as indicators of quality.

Third, for policyholders and consumer advocates, the approach shows what an explanation of a premium can contain beyond a list of influential factors. It states which controllable changes would lower the predicted premium, by how much, and how reliable that reduction is, so that a policyholder can weigh the benefit against the cost of making the change.

\subsection{Implications for Research}
\label{subsec:implications-research}

Our results contribute to three streams of research. For algorithmic accountability, they show that an explanation given to an affected individual needs to be reliable as well as valid and actionable, so that accountability to the individual reaches beyond explaining the model. For explainable AI and algorithmic transparency, they show that reliability can be enforced with a Conformal Prediction constraint, and that its price is a recommendation that can never ask less, and may ask more, of the policyholder. For research on AI-enabled insurance pricing, they show that most valid recommendations are not reliable and that the choice of technique matters.

Beyond insurance, the central idea of this paper is that the reliability of an explanation has to be enforced explicitly, whether inside the search or by filtering its output, because it does not follow automatically from the explanation being valid. The three axioms, Proposition~\ref{prop:coverage}, and Proposition~\ref{prop:cost} are formulated generally, and none of them depends on any property specific to insurance pricing. We therefore expect the underlying mechanism to extend, with minimal modification, to any setting with three structural features. An ML model makes a consequential, individual-level decision, some but not all of the inputs to that decision are within the affected individual's control, and predictive uncertainty varies across the input space. Algorithmic lending, automated hiring screens, and healthcare risk scores plausibly share this structure. We have not applied the approach in these domains, however, and this expectation remains a conjecture.

Section~\ref{sec:related-work} distinguished two layers of algorithmic accountability in insurance. The first is the risk-load allocation mechanism that \citet{zhangxu2024fairratemaking} study, and the second is the communication and recourse layer that this paper studies. A recourse approach that recommends changes requiring capital the policyholder does not have, such as an elevation retrofit, offers less effective recourse to lower-income policyholders. Given the correlations \citet{zhangxu2024fairratemaking} document between minority ownership and property characteristics, it may offer less effective recourse to minority policyholders as well, compared with others receiving structurally similar counterfactuals. We call this \emph{differential actionability}. Two policyholders can receive counterfactuals that are equally valid and equally reliable, yet unequally implementable, because the capital, time, or authority required to execute the recommended change is not distributed evenly across the population that receives it. Neither a fairness audit of the allocation mechanism, nor a technically reliable but generically actionable CE method, would surface it on its own, since each operates at only one of the two layers.

Two design principles follow from our application. Recommendations should be restricted to controllable features, and recourse should be constrained by where the model can be trusted, and not only by how close a change is to the original input. Whether these principles generalize outside insurance, and whether differential actionability recurs in those domains, is a topic for future research.

\subsection{Limitations and Future Research}
\label{sec:limitations}

The application in Section~\ref{sec:demonstration} is illustrative rather than a large-scale empirical validation. It shows how the CP constraint changes the recommendations for a few query policyholders and, for the telematics data, on average over a test sample, but it does not establish that these patterns hold across multiple application domains. A portfolio-scale analysis would test whether uncertainty-aware recommendations systematically demand larger or more numerous changes for policyholders with particular characteristics. 
This is the differential-actionability question raised in Section~\ref{subsec:implications-research}, and it requires a defensible proxy for policyholders' capacity to act. A behavioral study, for instance a vignette experiment comparing actionable CE-based explanations with conventional attribution-based ones on understandability, procedural fairness \citep{lindtyler1988}, and trust, would show whether policyholders find the recommendations more useful.

The reliability bound of Proposition~\ref{prop:coverage} concerns policyholder profiles drawn from the data distribution and accepted by the constraint. It does not certify the particular counterfactual returned, because $\boldsymbol{x}'_m$ is the output of an optimization procedure and not an independently drawn observation. Exact conditional guarantees are impossible under distribution-free, finite-sample requirements \citep{vovk2012conditional, barber2021limits}, so for an individual recommendation we describe the constraint as an uncertainty-screening device that the bound supports but does not certify. Closing this gap would require selection-aware conformal methods or a second, disjoint calibration split evaluated on the accepted region. The bound also does not apply to the soft-constraint fallback needed when the hard constraint is infeasible in high-uncertainty domains such as climate-related property risk, and it concerns the observed claim outcome and not a policyholder's unobserved underlying risk.

The approach also has limits in scope. It is applied to a frequency or severity model separately and not to the pure premium, so actionable changes are not yet expressed in premium terms. Recommending changes presumes that the association between actionable features and the outcome is causal, an assumption that the approach does not verify and that would benefit from insurer domain expertise. Robustness of explanations to small perturbations of the query instance, and the sensitivity of the recommendations to the target $y^u$, are not evaluated. We also treat plausibility as a desirable property, and formalizing it as a further axiom is a possible extension. Deployment raises further questions. Disclosing which features are actionable can invite gaming, since a policyholder told which values lower a premium may alter reported values without changing underlying risk. Communicating confidence levels to lay policyholders without overstating the property or overwhelming them with caveats remains an open interaction-design problem. In addition, cost-weighted proximity metrics that address differential actionability within Equation~\ref{eq:cp-ce} are a promising extension.

\section{Conclusion}
\label{sec:conclusion}

Algorithmic transparency in insurance pricing has so far been studied largely as a property of models, which features matter, or of allocation mechanisms, how collective risk should be fairly divided. We study a third layer, namely whether an individual affected by an algorithmic decision can understand it well enough to act, and whether the actions recommended to them are reliable given what the underlying model does and does not know. We adapt Counterfactual Explanation techniques to insurance frequency and severity models, state three axioms for actionable and reliable recourse, and meet the reliability axiom with a Conformal Prediction constraint that bounds the rate of unreliable recommendations.

In our flood-insurance and driver-telematics illustrations, the constrained recommendations are more reliable and are associated with larger predicted risk reductions, at the price of changing more features. Insurers can use these results to choose among techniques and to judge when a recommendation can be trusted. The bound concerns policyholder profiles accepted by the constraint and not the particular counterfactual returned, and closing this gap is the main open statistical question.

We also relate this individual-level problem to mechanism-level fairness research on insurance ratemaking. The two are complementary layers of algorithmic accountability, and whether recourse is equally actionable for everyone is a question for future research on algorithmic decision systems, in insurance and beyond.

\section*{Acknowledgments}
Agathe Fernandes Machado acknowledges that the project leading to this publication received funding from UNSW Business School, OBVIA, and UQAM. Fei Huang gratefully acknowledges support from the Australian Research Council Discovery Projects funding scheme (DP250104816 and DP260104297).

\paragraph{Data and Code Availability} Replication code is available at \url{https://github.com/fer-agathe/counterfactuals_insurance.git}. The extended R implementations of the CE techniques (Section~\ref{subsec:comparison}) are available at \url{https://github.com/fer-agathe/counterfactuals.git}, forked directly from the original package \url{https://github.com/dandls/counterfactuals.git}. The FIMA NFIP Redacted Claims dataset used in Section~\ref{subsec:demo-flood} and the synthetic driver-telematics dataset used in Section~\ref{subsec:demo-telematics} are publicly available (Appendices~\ref{sec:appendix-flood-data} and~\ref{sec:appendix-telematics}).

\clearpage
\appendix
\renewcommand{\thetable}{A\arabic{table}}\renewcommand{\theequation}{A\arabic{equation}}\renewcommand{\thefigure}{A\arabic{figure}}
\setcounter{table}{0}\setcounter{equation}{0}\setcounter{figure}{0}

\section{Interpretability Methods Compared}
\label{sec:ec-interpretability}

This appendix supports Section~\ref{subsec:why-ce}. It places Counterfactual Explanation (CE) within the taxonomy of ML interpretability methods and compares it with the other local, model-agnostic, post-hoc methods.

Figure~\ref{fig:ml-interpretability} places CE within the taxonomy of ML interpretability methods \citep{molnar2025}. Interpretability divides into models that are interpretable by design (e.g., generalized linear models (GLMs) and shallow decision trees) and complex models that require a separate analysis after training (post-hoc). We require a post-hoc method, because insurers are unlikely to replace validated, higher-performing pricing models with intrinsically interpretable ones solely to communicate with policyholders. Post-hoc methods are either tied to a specific model architecture (model-specific, e.g., attention weights in transformer-based pricing models) or applicable to any model (model-agnostic). We require a model-agnostic method, so that the approach generalizes across insurers and model choices. Model-agnostic methods are either global, describing model behavior averaged across a portfolio, or local, explaining a single prediction. We require a local method, because our objective is to explain one policyholder's premium to that policyholder. These three requirements do not single out CE, since many local, model-agnostic, post-hoc methods exist. Table~\ref{tab:interpretability-posthoc-agnostic} compares them and shows that contrastiveness and actionability are largely absent elsewhere in the toolkit that has been applied to insurance pricing models.

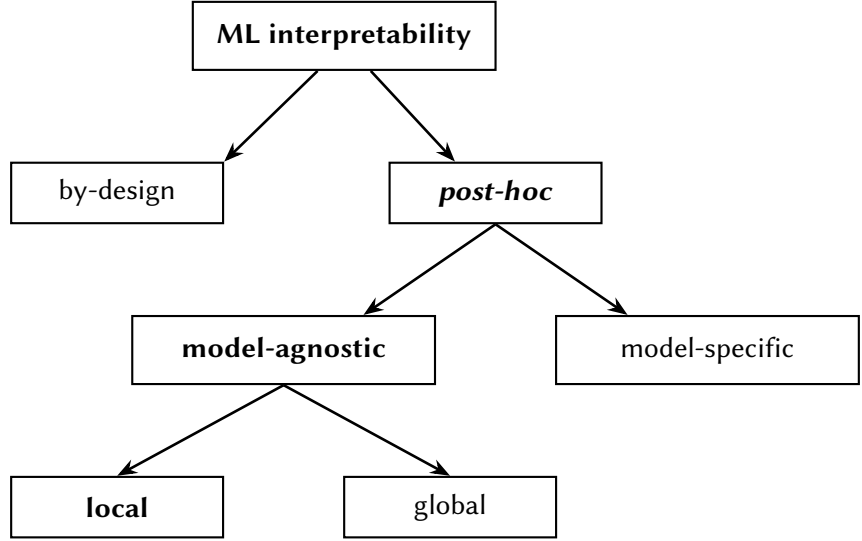
\begin{figure}[!htbp]
    \centering
\begin{tikzpicture}[
  box/.style={
    draw=black, thick,
    minimum width=40mm,
    minimum height=9mm,
    align=center,
    font=\sffamily
  },
  smallbox/.style={
    draw=black, thick,
    minimum width=28mm,
    minimum height=8mm,
    align=center,
    font=\sffamily
  },
  arr/.style={-Stealth, line width=1pt, black}
]
\node[box] (top) {\textbf{ML interpretability}};
\node[smallbox, below=12mm of top, xshift=-30mm] (by) {by-design};
\node[smallbox,     below=12mm of top, xshift=20mm]  (post) {\textit{\textbf{post-hoc}}};
\node[box, below=12mm of post, xshift=-28mm] (agn) {\textbf{model-agnostic}};
\node[box, below=12mm of post, xshift=28mm]  (spec) {model-specific};
\node[smallbox, below=12mm of agn, xshift=-22mm] (local) {\textbf{local}};
\node[smallbox, below=12mm of agn, xshift=22mm]  (global) {global};
\draw[arr] ($(top.south)+(-10pt,0)$) -- (by.north east);
\draw[arr] ($(top.south)+(10pt,0)$)  -- ($(post.north west)+(25pt,0)$);
\draw[arr] (post.south) -- ($(agn.north)+(30pt,0)$);
\draw[arr] (post.south) --  ($(spec.north)-(30pt,0)$);
\draw[arr] (agn.south) -- (local.north);
\draw[arr] (agn.south) --  (global.north);
\end{tikzpicture}
\caption{Taxonomy of machine learning interpretability methods, following \citet{molnar2025}. The bold path is the post-hoc, model-agnostic, local branch to which this paper's requirements narrow attention.}
    \label{fig:ml-interpretability}
\end{figure}

\begin{table}[!htbp]
\centering
\small
\begin{tabularx}{\textwidth}{p{2.1cm} p{1cm} p{1.5cm} X X}
\toprule
Method & Type & Category & Summary & Limitation for Policyholder Communication\\
\midrule
PDP \citep{friedman2001greedy} & global & plot & Marginal effect of a feature on model predictions & Reports model-level, not individual-level, information; assumes feature independence\\
ALE \citep{ale2020apley} & global & plot & Similar to PDP but accounts for feature correlations & Still model-level, not individual-level\\
PFI \citep{breiman2001random, JMLRfisher2019pfi} & global & importance & Change in model performance when a feature is permuted & Model-level; reports importance, not actionable direction, for any individual\\
ICE \citep{goldstein2015peeking} & local & plot & Individual-level analogue of PDP & Graphical, not contrastive; assumes feature independence; not easily read by lay users\\
LIME \citep{lime2016} & local & surrogate model & Locally approximates predictions with an interpretable model & Reports feature contributions, not the change needed to reach a different outcome; sampling methods can ignore feature dependencies\\
SHAP \citep{NIPS2017_8a20a862} & local & importance & Attributes a prediction to features via cooperative game theory & Answers ``why this value,'' not ``what would change it''; computationally costly; can mislead under feature dependence\\
CE \citep{wachter2018counterfactual} & local & contrastive & Minimal feature changes needed to obtain a different, preferred prediction & Multiple counterfactuals may exist per instance (addressed via multiplicity in Section~\ref{sec:ce-insurance})\\
\bottomrule
\end{tabularx}
\caption{Comparison of post-hoc, model-agnostic interpretability methods on their suitability for individual-level, actionable policyholder communication.}
\label{tab:interpretability-posthoc-agnostic}
\end{table}

As Table~\ref{tab:interpretability-posthoc-agnostic} shows, commonly used global interpretability methods, namely partial dependence plots (PDP), accumulated local effects (ALE), and permutation feature importance (PFI), are, by construction, model-level rather than individual-level, and therefore cannot answer a specific policyholder's question of what to change. Among local methods, individual conditional expectation (ICE), local interpretable model-agnostic explanations (LIME), and SHAP, each report some form of feature contribution or sensitivity for a single instance, but none is contrastive. Each describes the prediction actually received, not the smallest change required to obtain a different one, and none natively restricts itself to features the individual can control. CE is the only method in this comparison that directly targets both requirements. We therefore adopt it as the basis for the approach developed in Section~\ref{sec:ce-insurance}. As discussed next, however, CE was not originally designed for the regression, actionability-constrained, uncertainty-aware form the insurance setting requires.

The flood application illustrates the difference (Section~\ref{subsec:demo-flood}). For the query policyholder $\boldsymbol{x}_2$, SHAP and LIME report which features influence the predicted severity and in which direction (Tables~\ref{tab:shap_contributions} and~\ref{tab:lime_contributions} in Appendix~\ref{sec:appendix-flood-data}). Their lists include immutable features such as building age. The CE decision plot (Figure~\ref{fig:second-query-ex-flood}) instead shows which controllable features the policyholder can change to obtain a lower predicted severity. Section~\ref{subsec:ec-attribution-vs-ce} compares the attributions with the recommended changes for this policyholder.


\section{Flood Insurance Data}
\label{sec:appendix-flood-data}

\subsection{Data Access}

The \emph{FIMA NFIP Redacted Claims} dataset can be retrieved directly in R using the \href{https://github.com/ropensci/rfema}{\texttt{rfema}} package. Example notebooks for accessing and working with the dataset are available at \url{https://github.com/FEMA/openfema-samples}.

\subsection{Variable Description}

Table~\ref{tab:variables-flood} lists the eleven building and location characteristics retained as predictors for the severity model of Section~\ref{sec:demonstration}, together with which are treated as actionable. Descriptions of the original variables are available on the FEMA website. Preprocessing steps are shown in the project's code repository.

\begin{table}[!htbp]
\centering
\small
\begin{tabularx}{\textwidth}{|l|l|c|X|}
\hline
\textbf{Variable} & \textbf{Type} & \textbf{Actionability} & \textbf{Description} \\
\hline
basement & categ. & \checkmark & Type of the lowest enclosed area of the building. \\
elevation & categ. & \checkmark & Indicates whether the building has been constructed according to NFIP criteria for an elevated building, i.e., whether the lowest floor is raised above the expected flood level and supported by an appropriate foundation system. \\
contents\_loc & categ. & \checkmark & Location of the contents within the insured building. \\
nb\_floors & num. & $\times$ & Number of floors in the insured building. \\
non\_profit & categ. & $\times$ & Indicates whether the building is reported as a non-profit property in the policy application. \\
occupancy & categ. & $\times$ & Use and occupancy type of the insured structure. \\
primary\_residence & categ. & $\times$ & Indicates whether the building is a primary residence. \\
damage\_cause & categ. & $\times$ & Type of event that caused damage to the insured's property and contents. \\
rental & categ. & $\times$ & Indicates whether the property is a rental property. \\
building\_age & num. & $\times$ & Age of the insured building. \\
flood\_zone & categ. & $\times$ & NFIP Flood Zone derived from the Flood Insurance Rate Map (FIRM), used to rate the insured property. \\
\hline
\end{tabularx}
\caption{Description and actionability of the variables used in the flood claim severity study of Section~\ref{sec:demonstration}.}
\label{tab:variables-flood}
\end{table}

\subsection{SHAP values}

Table~\ref{tab:shap_contributions} illustrates the application of the SHAP interpretability method, presented in Table~\ref{tab:interpretability-posthoc-agnostic}, to the query instance $\boldsymbol{x}_2$. SHAP provides local feature contributions for the fitted XGBoost model on the flood insurance dataset. We report only the top-10 features based on the magnitude of their contribution. Immutable features cannot be excluded from the SHAP contribution calculation. Categorical variables are also represented using their one-hot encoded values (one indicator variable per category), which can make the resulting explanations difficult to interpret for non-technical users.

\begin{table}[!htbp]
\centering
\small
\begin{tabular}{lrr}
\toprule
\textbf{Feature} & \textbf{Contribution} & \textbf{Cumulative contribution} \\
\midrule
baseline & $\times$ & 0.3175 \\
\midrule
\textbf{basement \textit{0}} & -0.0407 & 0.2768 \\
building\_age & 0.0272 & 0.3040 \\
damage\_cause \textit{4} & 0.0227 & 0.3267 \\
\textbf{basement \textit{3}} & 0.0183 & 0.3450 \\
damage\_cause \textit{1} & 0.0174 & 0.3624 \\
\textbf{elevated\_building} & 0.0131 & 0.3756 \\
nb\_floors \textit{2} & 0.0106 & 0.3861 \\
\textbf{contents\_location \textit{Lowest\_floor\_and\_above}} & 0.0075 & 0.3936 \\
occupancy \textit{Mobile\_home} & -0.0064 & 0.3873 \\
primary\_residence & 0.0035 & 0.3907 \\
\midrule
prediction & $\times$ & 0.3996 \\
\bottomrule
\end{tabular}
\caption{SHAP feature contributions for the 10 most important features on $\boldsymbol{x}_2$. Positive contributions increase the predicted value and negative ones decrease it. Features in bold are actionable.}
\label{tab:shap_contributions}
\end{table}

\subsection{LIME}

Table~\ref{tab:lime_contributions} illustrates the application of the LIME interpretability method, presented in Table~\ref{tab:interpretability-posthoc-agnostic}, to the query instance $\boldsymbol{x}_2$. LIME generates perturbed observations around $\boldsymbol{x}_2$, evaluates the black-box model on these observations, and fits a local linear surrogate model. The feature weight reported in Table~\ref{tab:lime_contributions} corresponds to the coefficient associated with that feature in the local surrogate model. We report only the top-10 features based on the magnitude of their estimated coefficients in the linear model. As in the SHAP illustration in Table~\ref{tab:shap_contributions}, immutable features are not excluded, and categorical variables are represented using their one-hot encoded values, which can make the resulting explanations difficult to interpret for non-technical users.

\begin{table}[!htbp]
\centering
\small
\begin{tabular}{lr}
\toprule
\textbf{Feature} & \textbf{Feature weight} \\
\midrule
intercept & 0.1842 \\
\midrule
\textbf{basement \textit{3}} & 0.1935 \\
building\_age
& 0.0110 \\
\textbf{contents\_location \textit{Mobile\_home}}
& 0.0110 \\
nb\_floors \textit{3}
& -0.0102 \\
non\_profit & 0.0098 \\
flood\_zone \textit{V\_zone} & -0.0078 \\
flood\_zone \textit{Other\_A\_zone}
& -0.0069 \\
rental & -0.0065 \\
\textbf{basement \textit{1}} & 0.0054 \\ 
\textbf{basement \textit{2}} & 0.0048 \\
\midrule
prediction & 0.3984 \\
\bottomrule
\end{tabular}
\caption{LIME feature weights for the 10 most important features on $\boldsymbol{x}_2$. Positive weights increase the predicted value and negative ones decrease it. Features in bold are actionable.}
\label{tab:lime_contributions}
\end{table}

\subsection{Comparison with the counterfactual recommendations}
\label{subsec:ec-attribution-vs-ce}

Tables~\ref{tab:shap_contributions} and~\ref{tab:lime_contributions} can be set against the counterfactual recommendations for $\boldsymbol{x}_2$ in Table~\ref{tab:flood-query-instance}. Three points follow.

The attributions and the recommendations agree on which actionable feature matters most. Basement type has the largest SHAP contribution in magnitude ($-0.041$) and by far the largest LIME weight (0.19), and three of the four recommended counterfactuals change it. Elevation, which three of the four recommendations also change, appears in the SHAP list but not in the LIME list. The two attribution methods therefore do not agree on every actionable feature.

Much of both lists refers to features the policyholder cannot change. Six of the ten entries in each list are immutable or, for the cause of damage, unknown before a loss occurs. Building age is the largest positive SHAP contribution (0.027) and ranks second in the LIME list. The counterfactual search holds all of these features fixed.

The attributions do not translate into recommended changes. SHAP contributions are measured against the average prediction (0.32), not against an alternative configuration of the building. The four actionable entries in the SHAP list sum to about $-0.002$, so they give no indication that the predicted loss ratio can fall from 0.40 to between 0.34 and 0.17, as the four recommendations show. Neither list states which combination of changes reaches a given target, and neither indicates how reliable the resulting prediction is, which the CP bound in Table~\ref{tab:flood-query-instance} reports. The comparison is for one policyholder and is illustrative.


\section{Landscape of Counterfactual Explanation Techniques}
\label{sec:ec-landscape}

This appendix supports Section~\ref{subsec:comparison}.

Section~\ref{subsec:notation} separates three axioms, which are required, from a wider set of desirable properties, which are worth optimizing. Table~\ref{tab:ce-landscape} situates several representative techniques from the wider CE literature against the three axioms (validity, actionability, and reliability) and against the desirable properties most relevant to insurance pricing (plausibility, causality, and multiplicity). It also reports applicability to regression, handling of categorical features (common in insurance data), and the availability of a public implementation. Table~\ref{tab:ce-properties} later returns to the three axioms for the four techniques we retain. The properties in Table~\ref{tab:ce-landscape} are not jointly satisfied by most existing CE techniques, which were developed predominantly for classification (see \citealp{verma2020review, karimi2022survey} for comprehensive reviews of this literature).


\begin{table}[!htbp]
\centering
\scriptsize
\setlength{\tabcolsep}{3pt}
\begin{tabularx}{\textwidth}{l X X X X X X X X X}
\toprule
& \multicolumn{3}{c}{\textbf{Axioms}} & \multicolumn{3}{c}{\textbf{Desirable properties}} & \multicolumn{3}{c}{\textbf{Applicability}}\\
\cmidrule(lr){2-4}\cmidrule(lr){5-7}\cmidrule(lr){8-10}
Method & Valid. & Action. & Reliab. & Plausib. & Causal. & Multip. & Regr. & Categ. & Code\\
\midrule
FACE \citep{face2020} & \checkmark & \checkmark & $\times$ & \checkmark & $\times$ & $\times$ & $\times$ & $\checkmark$ & Python\\
Prototype \citep{interpretableCE2021} & \checkmark & \checkmark & $\times$ & \checkmark & $\times$ & $\times$ & $\times$ & \checkmark & Python\\
MOC \citep{Dandl_2020} & (\checkmark) & \checkmark & $\times$ & \checkmark & $\times$ & \checkmark & \checkmark & \checkmark & R\\
GeCo \citep{geco2021} & \checkmark & \checkmark & $\times$ & \checkmark & \checkmark & \checkmark & $\times$ & \checkmark & Python\\
Sequential AR \citep{Verma_Hines_Dickerson_2022} & \checkmark & \checkmark & $\times$ & \checkmark & \checkmark & \checkmark & $\times$ & \checkmark & Python\\
CounteRGAN \citep{pmlr-v180-nemirovsky22a} & \checkmark & \checkmark & $\times$ & \checkmark & $\times$ & $\times$ & $\times$ & \checkmark & Python\\
Consequential decisions \citep{pmlr-v108-karimi20a} & \checkmark & \checkmark & $\times$ & $\times$ & $\times$ & . & \checkmark & . & Python\\
Bayesian CE \citep{spooner2021counterfactualexplanationsarbitraryregression} & \checkmark & . & . & $\times$ & $\times$ & . & \checkmark & . & .\\
Plausible CE \citep{convexCE2020} & \checkmark & . & $\times$ & \checkmark & $\times$ & . & $\times$ & . & .\\
DiCE \citep{dice2020} & \checkmark &  \checkmark & $\times$ & $\times$ & $\times$ &   \checkmark &  \checkmark &  \checkmark & Python \\
DANCE \citep{Bobek2025ActionableAndDiverseCounterfactual} & \checkmark & . & . & . & . & . & $\times$ & . & Python \\
WhatIf \citep{8807255} & \checkmark & $\times$& $\times$ & \checkmark & $\times$ & $\times$ & $\times$ & . & .\\
NICE \citep{prototype2024brughmans} & \checkmark & $\times$ & $\times$ & \checkmark & $\times$ & $\times$ & $\times$ & . & . \\
\bottomrule
\end{tabularx}
\caption{Comparison of representative CE techniques from the wider literature against properties relevant to insurance pricing. Validity, actionability, and reliability correspond to Axioms~\ref{ax:validity}, \ref{ax:actionability}, and~\ref{ax:reliability}. Plausibility, causality, and multiplicity are desirable properties. Regression (``Regr.''), categorical variables (``Categ.''), and public code (``Code'') describe applicability. ``Categ.'' indicates whether the technique handles categorical variables. \checkmark\ indicates the property is satisfied, (\checkmark) that validity is obtained only through a post-hoc filter, $\times$ that it is not satisfied, and . that the information was not found or is unavailable in the source publication. For reliability, $\times$ means that we found no mechanism in the source publication that bounds the rate of unreliable recommendations.}
\label{tab:ce-landscape}
\end{table}

As Table~\ref{tab:ce-landscape} shows, techniques that incorporate plausibility (via causal structure or training-density proximity) are, with one exception, designed only for classification. To our knowledge, none of the techniques in the table bounds the rate of unreliable recommendations (Axiom~\ref{ax:reliability}). In the current implementation of WhatIf and NICE in the R package \texttt{counterfactuals} \citep{JSSv115i09}, both methods handle regression settings, but were originally designed for classification and still do not natively enforce actionability. DiCE \citep{dice2020} handles regression tasks but does not explicitly address the plausibility or causality constraints of the generated counterfactuals. Only MOC combines actionability, plausibility, multiplicity, and regression in its published form.

\section{Formulations and Extensions of the Four Techniques}
\label{sec:ec-formulations}

This appendix supports Section~\ref{subsec:comparison}. It gives the optimization problem solved by each of the four techniques and the extensions we make to WhatIf and NICE.

WhatIf restricts the search in Equation~\ref{eq:min-ce} to the empirical training distribution. Letting $\mathcal{D}_{\text{train}}' := \{\boldsymbol{x}' \in \mathcal{D}_{\text{train}} : h(\boldsymbol{x}') \in Y'\}$, it returns the $M$ closest such instances,
$$
\operatorname*{arg\,min}_{\boldsymbol{x}'_1,\ldots,\boldsymbol{x}'_M \in \mathcal{D}_{\text{train}}'} \sum_{m=1}^M \theta(\boldsymbol{x}'_m,\boldsymbol{x}).
$$
Every candidate is an observed instance from $\mathcal{D}_{\text{train}}$, which guarantees plausibility by construction.

NICE relaxes this by searching sequentially outward from a WhatIf-style initial candidate, changing one feature at a time toward the initial candidate's values and retaining the change that most improves a plausibility-based reward at each step, which also favors sparsity. 

MOC instead solves
$$
\arg \min_{\boldsymbol{x}' \in \mathbb{R}^d} l\left(h(\boldsymbol{x}'), Y'\right) + C_1 \cdot \theta(\boldsymbol{x}', \boldsymbol{x}) + C_2 \cdot \|\boldsymbol{x}'-\boldsymbol{x}\|_0 + C_3 \cdot \theta_k(\boldsymbol{x}', \mathcal{D}_{\text{train}}),
$$
via a genetic algorithm (NSGA-II), where the added terms directly penalize the number of changed features ($\|\cdot\|_0$, sparsity) and the Gower distance to $\boldsymbol{x}'$'s $k$ nearest training neighbors ($\theta_k$, plausibility), searching the continuous space $\mathbb{R}^d$ rather than $\mathcal{D}_{\text{train}}$, which lets it, unlike WhatIf, recommend changes not present in any observed policyholder. To ensure multiplicity, MOC allows the return of all non-dominated individuals across all generations of the genetic algorithm.

DiCE instead solves, for $M$ simultaneous counterfactuals,
$$
\arg\min_{\boldsymbol{x}'_1,\ldots,\boldsymbol{x}'_M \in \mathbb{R}^d} \:
\frac{1}{M}\sum_{m=1}^M l\left(h(\boldsymbol{x}'_m),Y'\right)
+\frac{C_1}{M}\sum_{m=1}^M \theta(\boldsymbol{x}'_m,\boldsymbol{x})
-C_2\det\left(K(\boldsymbol{x}'_1,\ldots,\boldsymbol{x}'_M)\right),
$$
where $K_{i,j}(\boldsymbol{x}') = (1 + \theta(\boldsymbol{x}'_i, \boldsymbol{x}'_j))^{-1}$. The determinant term is maximized when the $M$ recommendations are mutually dissimilar, directly targeting diversity rather than obtaining it as a byproduct of multiplicity, as WhatIf, NICE, and MOC each do in their own way. In practice, we use the DiCE method based on a genetic algorithm for counterfactual search. Unlike WhatIf and NICE, which ensure validity by construction, MOC and DiCE balance potentially conflicting objectives, so validity is not necessarily guaranteed. Nevertheless, the native implementations of both methods can be configured to return only valid counterfactuals. Table~\ref{tab:ce-properties} summarizes the resulting four techniques.

The base implementations of WhatIf and NICE do not natively enforce actionability. We extend their R implementations in the \texttt{counterfactuals} package \citep{JSSv115i09} to restrict candidates to those that leave immutable features unchanged relative to $\boldsymbol{x}$. The two extensions differ in mechanics because the two techniques search differently. To add actionability to WhatIf, we restrict the search to training instances matching the query instance's immutable categorical features exactly and, for immutable numerical features, lying within a user-specified tolerance $\epsilon$ of the query instance. We then overwrite the immutable numerical features to match $\boldsymbol{x}$ exactly and re-verify that the resulting instance's prediction still falls in $Y'$. A small $\epsilon$ favors plausibility, at the risk of returning no candidates when few training instances happen to share the query instance's immutable profile. When this occurs, we relax the match to instances sharing the largest number, rather than all, of the immutable categorical features.

NICE instead builds a counterfactual sequentially. It initializes from a nearby valid training instance and then, at each iteration, changes one additional feature of the query instance toward that instance's value, retaining the change that most improves a plausibility-based reward. To add actionability, we restrict the initial candidate search using the same immutable-feature matching used for WhatIf, so that the sequential search only ever moves through actionable features. 

\section{Conformal Prediction and Proof of Proposition~\ref{prop:coverage}}
\label{sec:appendix-cp}

This appendix reviews the split conformal prediction (SCP) method underlying the reliability constraint of Section~\ref{subsec:cp-constraint} and gives the full proof of Proposition~\ref{prop:coverage}. For a general treatment of conformal prediction beyond what is needed here, see \citet{vovk2022algorithmic} and the tutorial of \citet{ZaffranDieuleveut2024}.

\begin{definition}[Exchangeability]
A finite sequence of random variables $Z_1, \ldots, Z_n$ is \emph{exchangeable} if $(Z_1, \ldots, Z_n) \overset{d}{=} (Z_{\pi(1)}, \ldots, Z_{\pi(n)})$ for every permutation $\pi$ of $\{1, \ldots, n\}$. That is, the joint distribution is invariant to reordering.
\end{definition}

Let $\mathcal{D} := \{(\boldsymbol{X}_i, Y_i)\}_{i=1}^n$ be exchangeable draws from $(\boldsymbol{X}, Y)$, split at random into a training set $\mathcal{D}_{\text{train}}$ and a calibration set $\mathcal{D}_{\text{cal}}$. Given a model $h$ fitted on $\mathcal{D}_{\text{train}}$ and a conformity score $s : \mathcal{X} \times \mathcal{Y} \times \mathcal{Y}^{\mathcal{X}} \to \mathbb{R}$, let $\mathcal{S} = \{s(\boldsymbol{X}_i, Y_i, h) : (\boldsymbol{X}_i, Y_i) \in \mathcal{D}_{\text{cal}}\}$. For a new instance $(\boldsymbol{X}_0, Y_0) \sim (\boldsymbol{X}, Y)$ (for which $Y_0$ is unobserved), at miscoverage level $\alpha \in (0,1)$, the SCP interval is
\begin{equation}
    \widehat{C}_{1-\alpha}(\boldsymbol{X}_0) := \left\{y \in \mathcal{Y} : s(\boldsymbol{X}_0, y, h) \leq q_{1-\alpha}(\mathcal{S})\right\},
    \label{eq:appendix-predinterval}
\end{equation}
where $q_{1-\alpha}(\mathcal{S})$ is the empirical $(1-\alpha)$-quantile of $\mathcal{S}$. For any $(\boldsymbol{X}_0, Y_0) \sim (\boldsymbol{X}, Y)$ exchangeable with $\mathcal{D}_{\text{cal}}$, and assuming no ties in $\mathcal{S}$,
\begin{equation}
\label{eq:appendix-coverage}
1 - \alpha \; \leq \; \mathbb{P}\left(Y_0 \in \widehat{C}_{1-\alpha}(\boldsymbol{X}_0)\right) \; \leq \; 1 - \alpha + \frac{1}{\#\mathcal{D}_{\text{cal}} + 1}.
\end{equation}
When $h$ is a mean-regression model, as is the case for the frequency and severity models considered here, the standard score $s(\boldsymbol{x}, y, h) = |y - h(\boldsymbol{x})|$ yields an interval of constant width, $\widehat{C}_{1-\alpha}(\boldsymbol{X}_0) = [h(\boldsymbol{X}_0) \pm q_{1-\alpha}(\mathcal{S})]$.
\citet{lei2018distribution} instead scale residuals by an estimate of local dispersion $\hat{\sigma}$, using $s(\boldsymbol{x}, y, (h, \hat{\sigma})) = |y - h(\boldsymbol{x})| / \hat{\sigma}(\boldsymbol{x})$, giving the Locally Adaptive Conformal Prediction (LACP) interval
\begin{equation}
\label{eq:appendix-lacp}
\widehat{C}_{1-\alpha}(\boldsymbol{X}_0) = \left[\max\left(0,\, h(\boldsymbol{X}_0) - q_{1-\alpha}(\mathcal{S}) \, \hat{\sigma}(\boldsymbol{X}_0)\right),\; h(\boldsymbol{X}_0) + q_{1-\alpha}(\mathcal{S}) \, \hat{\sigma}(\boldsymbol{X}_0)\right],
\end{equation}
where the lower endpoint is truncated at 0 because insurance frequency and severity outcomes are non-negative. Because the outcome $Y_0$ itself is always non-negative, this truncation excludes no probability mass and leaves the coverage guarantee in Equation~\ref{eq:appendix-coverage} unaffected. The upper endpoint, which is the one relevant to the constraint in Equation~\ref{eq:cp-ce}, is untouched. We adopt the resulting interval because its width varies with $\boldsymbol{X}_0$.


It is wider in sparse, high-uncertainty regions and narrower in dense, well-determined ones. Both variants satisfy the marginal coverage guarantee of Equation~\ref{eq:appendix-coverage}, on average over the input space, but neither satisfies the stronger property of conditional coverage (coverage at $1-\alpha$ pointwise, for every $\boldsymbol{X}_0$). LACP provides a closer approximation to conditional coverage than the constant-width variant, which is why we adopt it in Section~\ref{subsec:cp-constraint}. It does not fully resolve possible violations of exchangeability between $\mathcal{D}_{\text{cal}}$ and the counterfactual instances generated by the CE search (Section~\ref{sec:limitations}).

\begin{proof}[Proof of Proposition~\ref{prop:coverage}]
Since $(\boldsymbol{X}, Y)$ is exchangeable with $\mathcal{D}_{\text{cal}}$, Equation~\ref{eq:appendix-coverage} applies to it and gives
$$
\mathbb{P}\left(Y \notin \widehat{C}_{1-\alpha}(\boldsymbol{X})\right) \leq \alpha.
$$
On the event $\{\boldsymbol{X} \in \mathcal{A}_\alpha,\; Y > y^u\}$ we have $Y > y^u \geq \sup(\widehat{C}_{1-\alpha}(\boldsymbol{X}))$ by definition of $\mathcal{A}_\alpha$, so $Y \notin \widehat{C}_{1-\alpha}(\boldsymbol{X})$. This event is therefore contained in $\{Y \notin \widehat{C}_{1-\alpha}(\boldsymbol{X})\}$, and
$$
\mathbb{P}\left(\boldsymbol{X} \in \mathcal{A}_\alpha,\; Y > y^u\right) \leq \mathbb{P}\left(Y \notin \widehat{C}_{1-\alpha}(\boldsymbol{X})\right) \leq \alpha.
$$
For the conditional form, whenever $\mathbb{P}(\boldsymbol{X} \in \mathcal{A}_\alpha) > 0$,
$$
\mathbb{P}\left(Y \leq y^u \mid \boldsymbol{X} \in \mathcal{A}_\alpha\right) = 1 - \frac{\mathbb{P}(\boldsymbol{X} \in \mathcal{A}_\alpha,\; Y > y^u)}{\mathbb{P}(\boldsymbol{X} \in \mathcal{A}_\alpha)} \geq 1 - \frac{\alpha}{\mathbb{P}(\boldsymbol{X} \in \mathcal{A}_\alpha)}.
$$
\end{proof}

\section{Proof of the Cost Proposition}
\label{sec:ec-cost-proof}

This appendix supports Section~\ref{subsec:cp-constraint}.

\begin{proof}[Proof of Proposition~\ref{prop:cost}]
The upper endpoint $h(\boldsymbol{x}') + q_{1-\alpha}(\mathcal{S})\hat\sigma(\boldsymbol{x}')$ is non-increasing in $\alpha$, because the empirical quantile $q_{1-\alpha}(\mathcal{S})$ is non-increasing in $\alpha$ and $\hat\sigma \geq 0$. A smaller $\alpha'$ therefore makes the constraint harder to satisfy, so $\mathcal{F}_{\alpha'} \subseteq \mathcal{F}_\alpha \subseteq (\mathbb{R}^d)^M$. The shared objective is minimized over a smaller set, which cannot decrease the optimum, and this gives both inequalities. Equality in the first holds if and only if the minimum over $\mathcal{F}_\alpha$ is attained at a point that also minimizes the objective over $(\mathbb{R}^d)^M$, that is, an unconstrained minimizer lies in $\mathcal{F}_\alpha$.
\end{proof}

\section{Counterfactual Explanation Hyperparameters}
\label{sec:ce-parameters}

Table~\ref{tab:parameters-ce} reports the values used for all hyperparameters of the CE techniques in both experiments in Section~\ref{sec:demonstration}, along with their justifications.

\begin{table}[!htbp]
\centering
\small
\begin{tabularx}{\textwidth}{|c|c|X|X|}
\hline
\textbf{Parameter} & \textbf{Approach} & \textbf{Description} & \textbf{Value} \\
\hline
$\epsilon$ & WhatIf & Neighborhood of immutable numerical features used to select instances for actionability & 0.3 (default value) \\
 $R$ & NICE & Reward function for generated counterfactuals & Plausibility criterion \\
 $\delta$ & NICE & Accepted margin for considering a prediction as correct & Half the median absolute error (default value)\\
 $k$ & MOC & Number of data points used for the plausibility objective & 1 (default value) \\
$\mu$ & MOC & Population size of the genetic algorithm & 5 \\
$n_{\text{gen}}$ & MOC, DiCE & Number of generations in the genetic algorithm & 10 \\
$p_{\text{mut}}$ & MOC & Probability of selecting an individual (resp. a feature) for mutation & 0.73 (resp. 0.5) (default values) \\
$p_{\text{rec}}$ & MOC & Probability of selecting an individual (resp. a feature) for recombination & 0.71 (resp. 0.62) (default values) \\
init. & MOC, DiCE & Population initialization strategy in the genetic algorithm & MOC: Feature values of new individuals are sampled from the respective feature value ranges and reset according to their importance measured by ICE (default value); DiCE: KD-tree (default value)\\
$C_1,C_2$ & DiCE & Weights for balancing proximity ($C_1$) and diversity ($C_2$) in the genetic algorithm optimization problem & $C_1 = 0.2$; $C_2 = 5$ (default values) \\
$t$ & DiCE & The genetic algorithm stops when the difference between the previous and current best loss is smaller than $t$ & 0.01 (default value)\\
$d$ & DiCE & Post-hoc plausibility filter retaining only counterfactuals within a bounded Gower distance of their nearest training neighbor & Top-5 counterfactuals \\
\hline
\end{tabularx}
\caption{Description and values of the hyperparameters used for the counterfactual explanation techniques in the experiments of Section~\ref{sec:demonstration}.}
\label{tab:parameters-ce}
\end{table}

\paragraph{WhatIf} The only hyperparameter affecting WhatIf is $\epsilon \in [0,1]$, which controls how close the candidate counterfactuals must be to the original instance with respect to the numerical immutable features. Therefore, $\epsilon=0.3$ allows numerical features to differ by at most $\pm 30\%$ from their original values. A higher $\epsilon$ defines a larger neighborhood, allowing candidate instances with more distant numerical immutable features to be selected, which may reduce the plausibility of the resulting counterfactuals when these features are subsequently set to the query instance's values. We recommend using the smallest $\epsilon$ that allows $\mathcal{D}'_{\text{train}} \neq \emptyset$ after accounting for the immutable features, thereby maximizing plausibility. This approach was adopted in the case study of Section~\ref{sec:demonstration}. The same $\epsilon$-neighborhood is currently applied to all numerical features, although using a feature-specific neighborhood could provide a better balance between the size of $\mathcal{D}'_{\text{train}}$ and the plausibility of the resulting counterfactuals. We leave such an extension for future work.

\paragraph{NICE} Through the iterations of the NICE algorithm, a reward function $R$ is used to identify the best valid counterfactual, based on sparsity, proximity, or plausibility. In the insurance context, proximity and plausibility are of primary interest. Since the initial counterfactual instance $\boldsymbol{x}^{(0)}$ is already selected to be close to the query instance and sparsity is ensured by changing one actionable feature at a time, we prioritize plausibility by choosing $R$ accordingly. The NICE implementation allows valid counterfactual instances to be selected for which the model's prediction error is sufficiently small in the regression setting, providing a first step toward more reliable counterfactuals. The parameter $\delta$ controls the maximum prediction error allowed for the model $h$ on the counterfactual instances, with the default value of the R package \texttt{counterfactuals} set to half the median absolute error. Since this parameter is not particularly relevant to our experiments, as we already encourage reliable changes through the uncertainty constraint, we use its default value.

\paragraph{MOC} Many hyperparameters are involved in the MOC implementation of the genetic algorithm. First, the population can be initialized in several ways, but we retain the default initialization (from the R package \texttt{counterfactuals}), as reported in Table~\ref{tab:parameters-ce}, since it does not have a substantial impact on the resulting counterfactuals in our experiments. Similarly, for $p_{\text{mut}}$ and $p_{\text{rec}}$, we retain the default values. 
In both the flood and telematics experiments, only three features are actionable, with a limited number of categories for each, so increasing these parameters does not substantially expand the search space of counterfactuals, as the entire set of potential counterfactuals that maximize plausibility can be readily reached. However, for datasets with a large number of actionable features, optimizing these hyperparameters with respect to different CE properties might result in substantially different counterfactuals. In addition, $\mu$ and $n_{\text{gen}}$, which control the population size and the number of generations of the genetic algorithm, respectively, were chosen to optimize the plausibility criterion while still obtaining multiple counterfactuals, given the low number of actionable features. Larger values of these parameters lead to more extensive exploration within the genetic algorithm and, consequently, a worse plausibility criterion. This resulted in $\mu = 5$ and $n_{\text{gen}} = 10$. We used the same values for the flood dataset because the number of actionable features is the same. As with the other genetic algorithm hyperparameters, when more actionable features are present in a dataset, these parameters can be tuned according to other criteria, such as sparsity and proximity. 

\paragraph{DiCE} The DiCE algorithm is implemented in the Python package \texttt{DiCE}. To adapt the CE techniques to the insurance context, we focused primarily on the R implementations of the three other techniques in \texttt{counterfactuals} and did not modify the core algorithm of DiCE. Whenever possible, we matched the corresponding genetic-algorithm parameters between DiCE and MOC, but only when they could be specified in the DiCE user function. Otherwise, we retained the default values. First, the population initialization for DiCE uses the default KD-tree algorithm, which selects counterfactual instances from $\mathcal{D}_{\text{train}}$, as no other initialization method matched that used in MOC. Similarly, we retained the default values for $C_1$, $C_2$, and $t$, for which there are no corresponding parameters in the MOC implementation. For $n_{\text{gen}}$, we use the value selected for MOC. Finally, $d$ is chosen to obtain at most the five closest counterfactuals to an observed instance in $\mathcal{D}_{\text{train}}$, thereby matching the plausibility objective with $k=1$ in the MOC optimization problem. Here, for our empirical illustration, obtaining five distinct counterfactuals per method is sufficient, but $d$ can be adjusted to select more counterfactuals depending on the needs of the application. As this is a post-hoc parameter, it can be easily modified without having to resolve the optimization problem in practice.

\section{Driver Telematics Data}
\label{sec:appendix-telematics}

\subsection{Data Access}

The synthetic driver-telematics dataset of \citet{so2021syntheticdatasetgenerationdriver} is available in the R package \texttt{CASdatasets} under the name \textit{catelematic13}.

\subsection{Variable Description}

Table~\ref{tab:variables-telematics} lists the nine explanatory features retained for the frequency model of Section~\ref{subsec:demo-telematics}, of which three are actionable.

\begin{table}[!htbp]
\centering
\small
\begin{tabularx}{\textwidth}{|l|l|c|X|}
\hline
\textbf{Variable} & \textbf{Type} & \textbf{Actionability} & \textbf{Description} \\
\hline
Duration & num. & $\times$ & Duration of the insurance coverage of the policy, in days. \\
Insured.age & num. & $\times$ & Age of the insured driver, in years. \\
Insured.sex & categ. & $\times$ & Sex of the insured driver. \\
Car.age & num. & $\times$ & Age of the vehicle, in years. \\
Credit.score & num. & $\times$ & Credit score of the insured driver. \\
Region & categ. & $\times$ & Type of region where the driver lives: rural or urban. \\
Pct.drive.rush & categ. & $\checkmark$ & Predominant rush-hour driving pattern: no rush-hour driving, mostly morning rush, mostly evening rush, or both. \\
Accel.14miles & categ. & $\checkmark$ & Whether the driver had at least one extreme acceleration event exceeding 14 mph/s per 1{,}000 miles. \\
Brake.14miles & categ. & $\checkmark$ & Whether the driver had at least one extreme braking event exceeding 14 mph/s per 1{,}000 miles. \\
\hline
\end{tabularx}
\caption{Description and actionability of the variables used in the telematics frequency study of Section~\ref{subsec:demo-telematics}. ``mph/s'' denotes miles per hour per second.}
\label{tab:variables-telematics}
\end{table}

\section{Further Notes on the Telematics Comparison}
\label{sec:ec-telematics-notes}

These notes supplement Section~\ref{subsubsec:comparison-telematics}. More generally, these results support using several CE techniques, because they produce counterfactuals with complementary properties and give practitioners a broader range of actionable recommendations. Most of the CEs proposed by NICE are already identified by WhatIf and MOC, and more than 70\% of those generated by MOC are identified by the other two techniques. NICE may therefore be unnecessary once WhatIf or MOC has been applied, and WhatIf is a practical default for large datasets. The results from DiCE are not shared by the other three techniques, so it can complement WhatIf.

The reported results should nevertheless be interpreted as empirical evidence from a single telematics dataset rather than as a universal ranking of CE methods. The relative performance of the methods may vary depending on the data characteristics, the predictive model, the number and type of actionable features, and the available computational budget. For this particular dataset, only three actionable features are considered, resulting in a restricted counterfactual search space. Genetic algorithms may therefore provide less additional benefit in this setting than in applications involving a larger number of actionable features. In this case, the randomness of the CE methods should be taken into account when assessing the stability of the generated CEs across multiple random seeds. NICE, MOC, and DiCE involve sampling features and observations from $\mathcal{D}_{\text{train}}$, whereas the WhatIf approach does not. WhatIf, NICE, and MOC are implemented in the same R package, \texttt{counterfactuals}, which makes their simultaneous application straightforward, whereas DiCE is implemented in Python.

\bibliography{biblio-isr}

\end{document}